\documentclass[accepted]{uai2026} 

\usepackage[american]{babel}

\usepackage{natbib} 
\usepackage{mathtools} 
\usepackage{booktabs} 
\usepackage{tikz} 
\usepackage{amssymb}
\usepackage{amsthm}
\usepackage{blindtext}
\usepackage{bm}
\usepackage{subcaption}
\usepackage{graphics}
\usepackage{amsmath}
\usepackage{algorithm}
\usepackage{algorithmic}
\usepackage{wrapfig}
\usepackage{etoc}
\usepackage{stfloats}

\newtheorem{thm}{Theorem}[section]

\newtheorem{lem}{Lemma}[section]

\newtheorem{assum}{Assumption}[section]

\title{Vanilla SGD with Momentum Survives Heavy-Tailed Noise: \\Convergence Analysis without Gradient Clipping or Normalization}

\author[1 *]{\href{mailto:<pecokai0204@gmail.com>?Subject=Your UAI 2026 paper}{Ryusei Yamada}{}}
\author[1 *]{\href{mailto:<naoki310303@gmail.com>?Subject=Your UAI 2026 paper}{Naoki Sato}{}}
\author[1]{Hideaki Iiduka}
\affil[1]{%
    Meiji University, Japan
}
\affil[*]{%
    Equal contribution
}
  
\begin{document}
\maketitle

\begin{abstract}
Stochastic gradient descent (SGD) is a cornerstone of modern optimization. While its performance under heavy-tailed noise is often addressed through specialized modifications such as gradient clipping or normalization, we investigate a more fundamental question: how does vanilla SGD, particularly with momentum, perform in the presence of heavy-tailed noise? In this paper, we refine existing convergence results for vanilla SGD and, more importantly, provide the first comprehensive convergence analysis of vanilla SGD with momentum for strongly convex, convex, and nonconvex objectives, without employing any gradient control mechanisms. Our results demonstrate that the obtained convergence rates are inferior to the optimal rates achieved by clipped or normalized variants of SGD, thereby revealing inherent limitations of vanilla methods under heavy-tailed noise. The theoretical findings are supported by experiments on synthetic functions.
\end{abstract}

\addtocontents{toc}{\protect\setcounter{tocdepth}{-1}}
\section{Introduction}
\label{sec:intro}
This paper considers the following optimization problem:
\begin{align*}
\min_{\bm{x} \in \mathbb{R}^d} f(\bm{x}) := \frac{1}{n}\sum_{i=1}^{n}f_i(\bm{x}),
\end{align*}
where each $f_i$ ($i \in [n]$) is differentiable. We address three fundamental cases where $f$ is strongly convex, convex, or nonconvex. Empirical risk minimization is a paramount objective that lies at the core of machine learning. As primary solvers for this problem, stochastic gradient descent (SGD) \citep{Her1951Ast} and its momentum-based variant \citep{Polyak1964Som, Rumelhart1986Lea} remain the gold standard in contemporary optimization. A standard assumption in the convergence analysis of such stochastic algorithms is bounded variance of stochastic noise, specifically, that the second moment of the error between the stochastic and full gradients is bounded. While numerous seminal results \citep{Nemirovski2009Rob, Rakhlin2012Mak, Liu2020AnI} rely on this assumption, recent experimental evidence suggests that stochastic noise in modern deep learning \citep{Simsekli2019ATa, Battash2024Rev, Ahn2024Lin} and reinforcement learning \citep{Garg2021OnP} often exhibits heavy tails. Consequently, research has shifted focus toward the convergence behavior of algorithms when the bounded variance assumption is violated, specifically, when stochastic noise possesses only bounded $\mathfrak{p}$-th moments for $\mathfrak{p} \in (1, 2]$.

Under this bounded $\mathfrak{p}$-th moment assumption, the convergence of vanilla SGD is known to break down \citep{Zhang2020Why}. To mitigate this issue, clipped SGD, which scales the gradient magnitude to stay within a predefined threshold, has emerged as a robust alternative. Clipped SGD and its variants have been proven to converge in both expectation \citep{Zhang2020Why} and with high probability \citep{Cutkosky2021Hig, Liu2023Bre, Sadiev2023Hig, Nguyen2023Imp, Nguyen2023Hig, Liu2024Hig}. Notably, these studies often incorporate gradient normalization (see Table \ref{tab:1}), and it has been widely considered essential to clip or normalize gradients to ensure convergence when noise variance is unbounded.

Recently, several studies have begun to challenge this notion by exploring the convergence of SGD (with momentum) without clipping under bounded $\mathfrak{p}$-th moment assumptions. \citet{Hubler2025Fro} demonstrated that normalized SGD converges with high probability even without clipping. \citet{Hubler2025Fro} and \citet{Liu2025Non} were the first to establish that normalized SGD with momentum converges in expectation, achieving rates of $\mathcal{O}\left( T^{-\frac{\mathfrak{p}-1}{3\mathfrak{p}-2}}\right)$ when the tail index $\mathfrak{p}$ is known \citep{Hubler2025Fro, Liu2025Non}, and $\mathcal{O}\left( T^{-\frac{\mathfrak{p}-1}{2\mathfrak{p}}}\right)$ when it is unknown \citep{Liu2025Non}. Most recently, \citet{Fatkhullin2025Can} and \citet{He2025Acc} proved that vanilla SGD without any normalization or clipping can indeed converge in expectation. Inspired by these recent developments, we decided to investigate the convergence rate of vanilla SGD with momentum under heavy-tailed noise.

\begin{table*}[!t]
\centering
\caption{Comparison of assumptions, algorithmic features, and convergence rates between the related and our work under heavy-tailed noise. The abbreviations "w.o. Clip.", "w.o. Norm.", and "w. Mom." denote "without gradient clipping," "without gradient normalization," and "with momentum," respectively.} 
\resizebox{\textwidth}{!}{%
\begin{tabular}{ccccccc}
\toprule
\textbf{Reference} & \textbf{Function} & \textbf{Smoothness} & \textbf{w.o. Clip.} & \textbf{w.o. Norm.} & \textbf{w. Mom.} & \textbf{Rates} \\ 
\midrule
\citet{Zhang2020Why} &S.C. & $L$-smooth & $\times$ & $\times$ & $\times$ & $\mathcal{O}(T^{-\frac{2(\mathfrak{p}-1)}{\mathfrak{p}}})$ \\
\midrule
\citet{Sadiev2023Hig} &S.C. & $L$-smooth & $\times$ & \checkmark & $\times$ & $\mathcal{O}(T^{-\frac{2(\mathfrak{p}-1)}{\mathfrak{p}}})$\\
\midrule
\citet{Fatkhullin2025Can} &S.C. & H\"{o}lder-smooth & \checkmark & \checkmark &$\times$ & $\mathcal{O}(T^{-(\mathfrak{p}-1)})$\\
\midrule
\textbf{Ours} (Theorem \ref{thm:sgd_s}) &S.C. & H\"{o}lder-smooth & \checkmark & \checkmark & $\times$ & $\mathcal{O}\left( T^{-(\mathfrak{p}-1)}\right)$\\
\midrule
\textbf{Ours} (Theorem \ref{thm:sgdm_s}) &S.C. & H\"{o}lder-smooth & \checkmark & \checkmark & \checkmark & $\mathcal{O}\Big( (1-\eta)^T + \eta^{\mathfrak{p}}\Big)$\\
\midrule\midrule
\citet{Sadiev2023Hig} &C.& $L$-smooth & $\times$ & \checkmark & $\times$ & $\mathcal{O}(T^{-\frac{\mathfrak{p}-1}{\mathfrak{p}}})$\\
\midrule
\citet{Fatkhullin2025Can} &C.& H\"{o}lder-smooth & \checkmark & \checkmark &$\times$ & $\mathcal{O}(T^{-\frac{\mathfrak{p}-1}{\mathfrak{p}}})$ \\
\midrule
\textbf{Ours} (Theorem \ref{thm:sgd_c}) &C.& H\"{o}lder-smooth & \checkmark & \checkmark & $\times$ & $\mathcal{O}\Big( \frac{\log T}{T^{(\mathfrak{p}-1)/\mathfrak{p}}}\Big)$ \\
\midrule
\textbf{Ours} (Theorem \ref{thm:sgdm_c}) &C.& H\"{o}lder-smooth & \checkmark & \checkmark & \checkmark & $\mathcal{O}\left( \frac{1}{\eta T} + \eta^{\mathfrak{p-1}}\right)$ \\
\midrule\midrule
\citet{Zhang2020Why} &N.C. & $L$-smooth & $\times$ & $\times$ & $\times$ & $\mathcal{O}(T^{-\frac{\mathfrak{p}-1}{3\mathfrak{p}-2}})$ \\
\midrule
\citet{Cutkosky2021Hig} &N.C. & $L$-smooth & $\times$ & $\times$ & \checkmark & $\mathcal{O}(T^{\frac{\mathfrak{p}-1}{3\mathfrak{p}-2}})$\\
\midrule
\citet{Liu2023Bre} &N.C. & $L$-smooth & $\times$ & $\times$ & \checkmark & $\mathcal{O}(T^{-\frac{\mathfrak{p}-1}{3\mathfrak{p}-2}})$\\
\midrule
\citet{Sadiev2023Hig} &N.C. & $L$-smooth & $\times$ & \checkmark & $\times$ & $\mathcal{O}(T^{-\frac{\mathfrak{p}-1}{\mathfrak{p}}})$\\
\midrule
\citet{Nguyen2023Hig} &N.C. & $L$-smooth & $\times$ & $\times$ & $\times$ & $\mathcal{O}(T^{-\frac{\mathfrak{p}-1}{3\mathfrak{p}-2}})$\\
\midrule
\citet{Hubler2025Fro} &N.C. & $L$-smooth & \checkmark & $\times$ & \checkmark& $\mathcal{O}(T^{-\frac{\mathfrak{p}-1}{3\mathfrak{p}-2}})$\\
\midrule
\citet{Liu2025Non} &N.C. & $(L_0,L_1)$-smooth & \checkmark & $\times$ & $\checkmark$ & $\mathcal{O}(T^{-\frac{\mathfrak{p}-1}{3\mathfrak{p}-2}})$ \\
\midrule
\citet{Fatkhullin2025Can} &N.C. & H\"{o}lder-smooth & \checkmark & \checkmark &$\times$ & $\mathcal{O}(T^{-\frac{\mathfrak{p}-1}{\mathfrak{p}}})$\\
\midrule
\textbf{Ours} (Theorem \ref{thm:sgd_n}) &N.C. & H\"{o}lder-smooth & \checkmark & \checkmark & $\times$ & $\mathcal{O}\Big( \frac{\log T}{T^{(\mathfrak{p}-1)/\mathfrak{p}}}\Big)$\\
\midrule
\textbf{Ours} (Theorem \ref{thm:sgdm_n})&N.C. & H\"{o}lder-smooth & \checkmark & \checkmark & \checkmark & $\mathcal{O}\Big( T^{-\frac{\mathfrak{p}-1}{2\mathfrak{p}}}\Big)$\\
\bottomrule
\end{tabular}
}
\label{tab:1}
\end{table*}

Our contribution can be summarized as follows:
\begin{enumerate}[leftmargin=*,itemsep=0pt, topsep=0pt]
\item We provide the first theoretical guarantee that vanilla SGD with momentum converges under heavy-tailed noise for strongly convex, convex, and nonconvex functions. Specifically, for nonconvex objectives, we establish a convergence rate of $\mathcal{O}\left( T^{-\frac{\mathfrak{p}-1}{2\mathfrak{p}}}\right)$ when the tail index $\mathfrak{p} \in (1,2]$ is known, and $\mathcal{O}\left( T^{-\frac{\mathfrak{p}-1}{4}}\right)$ when it is unknown (see Theorem \ref{thm:sgdm_n}). While these rates are slightly inferior to those of normalized SGD with momentum \citep{Liu2025Non}, our findings reveal how well vanilla SGD with momentum performs under heavy-tailed noise, serving as a fundamental baseline for future algorithmic improvements.
\item We establish that vanilla SGD converges in expectation under heavy-tailed noise across all three classes of objective functions. In contrast to existing work \citep{Fatkhullin2025Can}, our results do not require the assumption of bounded gradients and hold under significantly weaker conditions, offering a more general and robust theoretical framework.
\item Through experiments on synthetic functions, we demonstrate that the condition $\nu + 1 \leq \mathfrak{p}$ is crucial for stable convergence of vanilla SGD with momentum. Here, $\nu \in (0,1]$ denotes the H\"{o}lder continuity parameter (Assumption \ref{assum:01}) and $\mathfrak{p} \in (1,2]$ is the tail index (Assumption \ref{assum:02}). This condition highlights a vital alignment between our theory and empirical observations (Section~\ref{sec:exp}). Notably, since standard $L$-smoothness ($\nu=1$) fails to satisfy this condition when noise is strictly heavy-tailed ($\mathfrak{p}<2$), our results suggest that H\"{o}lder smoothness is a more appropriate and versatile framework for analysis in such settings.
\end{enumerate}

\section{Preliminaries}
\label{sec:pre}
\paragraph{Notation} Let $\mathbb{N}$ be the set of non-negative integers. For $m \in \mathbb{N} \setminus \{0\}$, define $[m] := \{1,2,\ldots,m\}$. Let $\mathbb{R}^d$ be a $d$-dimensional Euclidean space with inner product $\langle \cdot, \cdot \rangle$, which induces the norm $\| \cdot \|$. Let $(\bm{x}_t)_{t \geq 1} \subset \mathbb{R}^d$ denote the sequence of iterates generated by the recursion in Algorithm~\ref{alg:1}. Let $\xi$ be a random variable, and let $\mathbb{E}_{\xi}[X]$ denote the expectation with respect to $\xi$ of a random variable $X$. $\xi_{t,i}$ is a random variable generated from the $i$-th sampling at time $t$, and $\bm{\xi}_t := (\xi_{t,1}, \xi_{t,2}, \ldots, \xi_{t,b})^\top$ is independent of $(\bm{x}_k)_{k=0}^{t}$, where $b\in [n]$ is the batch size. The independence of $\bm{\xi}_1, \bm{\xi}_2, \ldots$ allows us to define the total expectation $\mathbb{E}$ as $\mathbb{E}=\mathbb{E}_{\bm{\xi}_1} \mathbb{E}_{\bm{\xi}_2} \cdots \mathbb{E}_{\bm{\xi}_t}$. Let $\mathsf{G}_{\xi}(\bm{x})$ be the stochastic gradient of $f(\cdot)$ at $\bm{x} \in \mathbb{R}^d$. A mini-batch $\mathcal{S}_t$ consists of $b$ samples at time $t$, and the mini-batch stochastic gradient of $f(\bm{x})$ for $\mathcal{S}_t$ is defined as $\nabla f_{\mathcal{S}_t}(\bm{x}) := \frac{1}{b} \sum_{i \in [b]} \mathsf{G}_{\xi_{t,i}}(\bm{x}) = \frac{1}{b}\sum_{i \in \mathcal{S}_t} \nabla f_i (\bm{x}).$ We denote $\bm{x}^{\star} \in \operatorname*{argmin}_{\bm{x} \in \mathbb{R}^d} f(\bm{x})$ and $f^{\star} := f(\bm{x}^{\star})$ in the strongly convex and convex cases.

We will refer to SGD with momentum that employs neither normalization nor clipping as vanilla SGD with momentum. Furthermore, we will refer to vanilla SGD without momentum obtained by setting $\beta=0$ in Algorithm \ref{alg:1} simply as vanilla SGD. 
\begin{algorithm}[H]%
\caption{Vanilla SGD with Momentum}%
\begin{algorithmic}\label{alg:1}
\REQUIRE$\bm{x}_1 \in \mathbb{R}^d, \eta_t> 0, b \in [n], \beta \in [0, 1), \bm{m}_0 := \bm{0}, T \in \mathbb{N}$
\FOR{$t=1$ to $T$}
\STATE{$\bm{m}_t := \beta \bm{m}_{t-1} + (1-\beta) \nabla f_{\mathcal{S}_t}(\bm{x}_t)$}
\STATE{$\bm{x}_{t+1} := \bm{x}_t - \eta_t \bm{m}_t$}
\ENDFOR
\RETURN $\bm{x}_{T+1}$
\end{algorithmic}
\end{algorithm}

We make the following assumptions.
\begin{assum}\label{assum:01}
$\nabla f \colon \mathbb{R}^d \to \mathbb{R}^d$ is H\"{o}lder continuous; i.e., there exist $\nu \in (0,1]$ and $L>0$ such that, for all $\bm{x}, \bm{y} \in \mathbb{R}^d$, 
\begin{align*}
\| \nabla f(\bm{x}) - \nabla f(\bm{y}) \| \leq L \| \bm{x} - \bm{y} \|^{\nu}.
\end{align*}
\end{assum}

\begin{assum}\label{assum:02}
(i) For all $\bm{x} \in \mathbb{R}^d$ that does not depend on $\xi$, 
$$ \mathbb{E}_{\xi} [\mathsf{G}_{\xi}(\bm{x})] = \nabla f(\bm{x}).$$
(ii) There exist $\sigma \geq 0$ and $\mathfrak{p} \in (1, 2]$ such that, for all $\bm{x} \in \mathbb{R}^d$ that does not depend on $\xi$,
$$ \mathbb{E}_{\xi} [\|\mathsf{G}_{\xi}(\bm{x}) - \nabla f(\bm{x})\|^\mathfrak{p}] \leq \sigma^{\mathfrak{p}}.$$
\end{assum}
Assumption \ref{assum:01} forms the core of our analysis. Since it reduces to standard $L$-smoothness when $\nu=1$, this assumption represents an extension of $L$-smoothness. Assumption~\ref{assum:02}(ii) is a standard condition for analyzing optimizers under heavy-tailed noise and is widely used in previous studies \citep{Zhang2020Why, Cutkosky2021Hig, Sadiev2023Hig, Nguyen2023Imp, Chezhegov2025Cli}. When $\mathfrak{p}=2$, it reduces to the standard bounded variance assumption \citep{Nemirovski2009Rob, Ghadimi2012Opt, Ghadimi2013Sto}; when $\mathfrak{p} < 2$, the stochastic gradient may possess unbounded variance.

\subsection{Tools for heavy-tailed analysis}
\label{sec:tool}
These lemmas are useful for analysis under heavy-tailed noise. 
For the sake of completeness, we include their proofs in Appendix \ref{sec:proof_A}.
\begin{lem}\label{lem:holder}
Suppose that Assumption \ref{assum:01} holds. Then, for all $\bm{x}, \bm{y} \in \mathbb{R}^d$, the following inequality holds:
\begin{align*}
f(\bm{y}) \leq f(\bm{x}) + \langle \nabla f(\bm{x}), \bm{y}-\bm{x} \rangle + \frac{L}{\nu + 1} \| \bm{y} - \bm{x} \|^{\nu+1}.
\end{align*}
\end{lem}
\begin{lem}\label{lem:q-norm}
Let $q \in (1,2]$. For any $\bm{y} \in \mathbb{R}^d$ and any $\bm{x} \in \mathbb{R}^d \setminus \{\bm{0}\}$, the following inequality holds:
\begin{align*}
\| \bm{x} \pm \bm{y} \|^q
\leq \| \bm{x} \|^q \pm q\| \bm{x} \|^{q-2}\langle \bm{x}, \bm{y} \rangle + 2^{2-q}\| \bm{y} \|^q.
\end{align*}
\end{lem}
\begin{lem}\label{lem:moment}
Suppose that Assumption \ref{assum:02} holds. Then, for all $t \in \mathbb{N}$ and all $\bm{x} \in \mathbb{R}^d$ that does not depend on $\bm{\xi}_t$, the following inequality holds:
\begin{align*}
\mathbb{E}_{{\bm{\xi}}_t}[\|\nabla f_{{\mathcal{S}}_t}(\bm{x}) - \nabla f(\bm{x})\|^{\mathfrak{p}}\big] \le \frac{C_{\mathfrak{p}}\sigma^{\mathfrak{p}}}{b^{\mathfrak{p}-1}},
\end{align*}
where $C_\mathfrak{p} \in [1,2)$ is the constant given in Lemma \ref{lem:vBE_vector}.
\end{lem}
\begin{lem}\label{lem:exp}
Let $\mathfrak{p} \in (1,2]$ and $\nu \in (0,1]$, suppose that $\nu+1 \leq \mathfrak{p}$. Then, for all random vectors $\bm{X} \in \mathbb{R}^d$, the following inequality holds:
\begin{align*}
\mathbb{E}_{\bm{X}}[\|\bm{X}\|^{\nu+1}] \leq (\mathbb{E}_{\bm{X}}[\|\bm{X}\|^\mathfrak{p}])^{\frac{\nu+1}{\mathfrak{p}}}.
\end{align*}
\end{lem}

Lemma \ref{lem:holder} characterizes functions with H\"{o}lder continuous gradients. As it recovers the classical descent lemma when $\nu=1$, it can be viewed as a generalized descent lemma (see, e.g., \citep{Nesterov2015Uni, Maryam2016Ont}). Lemma~\ref{lem:q-norm} has been extensively discussed in functional analysis \citep{Beauzamy1982Int, Xu1991Cha, Chidume2009Geo, Rodomanov2020Smo}; for $q=2$, it yields the standard expansion of the squared Euclidean norm, $\|\bm{x} \pm \bm{y}\|^2 = \|\bm{x}\|^2 \pm 2\langle \bm{x}, \bm{y} \rangle + \|\bm{y}\|^2$. Lemma \ref{lem:moment} is obtained using von Bahr-Esseen's inequality \citep{Bahr1965Ine} (Lemma \ref{lem:vBE_vector}). Notably, it reduces to the standard result $\mathbb{E}_{\bm{\xi}_t} [ \| \nabla f_{\mathcal{S}_t}(\bm{x}) - \nabla f(\bm{x}) \|^2 ] \leq \frac{\sigma^2}{b}$ when $\mathfrak{p}=2$. Lemma \ref{lem:exp} is crucial for our main results; it dictates the key constraint $\nu+1 \leq \mathfrak{p}$ required for our convergence analysis.

\section{Main Results}
\label{sec:main}
\subsection{Incompatibility between \texorpdfstring{$L$}{L}-smoothness and heavy-tailed noise}
$L$-smoothness is a standard assumption in the analysis of first-order optimization algorithms ($\bm{x}_{t+1} := \bm{x}_t - \eta_t \bm{d}_t$). It allows the use of the classical descent lemma, which is a special case of Lemma \ref{lem:holder} with $\nu = 1$:
\begin{align*}
&f(\bm{x}_{t+1}) \\
&\quad\leq f(\bm{x}_t) + \langle \nabla f(\bm{x}_t), \bm{x}_{t+1} - \bm{x}_t\rangle + \frac{L}{2}\| \bm{x}_{t+1} - \bm{x}_t \|^2 \\
&\quad= f(\bm{x}_t) - \eta_t \langle \nabla f(\bm{x}_t), \bm{d}_t \rangle + \frac{L\eta_t^2}{2}\| \bm{d}_t \|^2.
\end{align*}
Whether analyzing SGD ($\bm{d}_t := \nabla f_{\mathcal{S}_t}(\bm{x}_t)$) or any other stochastic variant, one cannot avoid the need to upper-bound $\mathbb{E}[\| \bm{d}_t \|^2]$. Under the standard bounded variance assumption, this is straightforward: $\mathbb{E}\left[ \| \nabla f_{\mathcal{S}_t}(\bm{x}_t) \|^2 \right] \leq \frac{\sigma^2}{b} + \mathbb{E}\left[\| \nabla f(\bm{x}_t) \|^2 \right]$. However, when stochastic noise exhibits heavy-tailed behavior (Assumption \ref{assum:02}) and the variance is unbounded, $\|\nabla f_{\mathcal{S}_t}(\bm{x}_t) \|^2$ cannot be bounded, causing the standard analysis to collapse. Consequently, convergence for vanilla SGD or SGD with momentum under heavy-tailed noise has been considered intractable under $L$-smoothness, leading prior studies to rely on clipping or normalization (see Section \ref{sec:intro}).

To circumvent this problem, we utilize the H\"{o}lder continuity of the gradient (Assumption \ref{assum:01}). Applying Lemma \ref{lem:holder} yields:
\begin{align*}
f(\bm{x}_{t+1})
&\leq f(\bm{x}_t) - \eta_t \langle \nabla f(\bm{x}_t), \bm{d}_t \rangle + \frac{L\eta_t^{\nu+1}}{\nu+1}\| \bm{d}_t \|^{\nu+1}.
\end{align*}
Crucially, this shift of the framework requires an upper-bound estimate for $\| \nabla f_{\mathcal{S}_t}(\bm{x}_t) \|^{\nu+1}$ rather than the squared norm. By leveraging Lemmas \ref{lem:q-norm}--\ref{lem:exp}, we obtain the following key auxiliary lemma, which serves as the technical foundation for our main results.

\begin{lem}\label{lem:aux}
Suppose that Assumptions \ref{assum:01} and \ref{assum:02} hold. Then, for all $\bm{x} \in \mathbb{R}^d$ that does not depend on $\bm{\xi}_t$, and all $t \in \mathbb{N}$, the following holds:
\begin{align*}
&\mathbb{E}_{{\bm{\xi}}_t} \left[\|\nabla f_{\mathcal{S}_t}(\bm{x})\|^{\nu+1}\right] \\
&\quad\leq 2^{1-\nu} \left( \frac{C_\mathfrak{p} \sigma^\mathfrak{p}}{b^{\mathfrak{p}-1}} \right)^{\frac{\nu+1}{\mathfrak{p}}} +  \| \nabla f(\bm{x}) \|^{\nu+1}  \\
&\quad\leq 2^{1-\nu} \left( \frac{C_\mathfrak{p} \sigma^\mathfrak{p}}{b^{\mathfrak{p}-1}} \right)^{\frac{\nu+1}{\mathfrak{p}}} + \frac{\nu+1}{2} \| \nabla f(\bm{x}) \|^2 + \frac{1-\nu}{2}.
\end{align*}
\end{lem}
Notably, Lemma \ref{lem:aux} recovers the standard bound $\mathbb{E}_{\bm{\xi}_t}\left[ \| \nabla f_{\mathcal{S}_t}(\bm{x}) \|^2 \right] \leq \frac{\sigma^2}{b} + \| \nabla f(\bm{x}) \|^2 $ when $\nu=1$ and $\mathfrak{p}=2$. This generalization demonstrates that existing analyses under $L$-smoothness and bounded variance can be extended to the heavy-tailed regime by adopting the H\"{o}lder smooth framework and the tools presented in Section \ref{sec:tool}.

Due to space constraints, we will define the following constant frequently appearing in our theorems:
\begin{align}\label{eq:E0}
E_0 := 2^{1-\nu} \left( \frac{C_\mathfrak{p} \sigma^\mathfrak{p}}{b^{\mathfrak{p}-1}} \right)^{\frac{\nu+1}{\mathfrak{p}}} + \frac{1-\nu}{2}.
\end{align}

\subsection{Strongly Convex Case}
We introduce a key property that holds under Assumption~\ref{assum:01} and $\mu$-strong convexity. Specifically, for all $\bm{x}, \bm{y}$ in the domain where $f$ is $\mu$-strongly convex and H\"{o}lder smooth, we have:
\begin{align*}
\mu\| \bm{x} - \bm{y} \|^2
&\leq\langle \nabla f(\bm{x}) - \nabla f(\bm{y}), \bm{x} - \bm{y} \rangle \\
&\leq \| \nabla f(\bm{x}) - \nabla f(\bm{y}) \| \| \bm{x} - \bm{y} \| \\
&\leq L\| \bm{x} - \bm{y} \|^{\nu+1}.
\end{align*}
That is, we have $\| \bm{x} - \bm{y} \|^{1-\nu} \leq \frac{L}{\mu}$ for all $\bm{x}, \bm{y}$ in the domain where $f$ is $\mu$-strongly convex and H\"{o}lder smooth ($\bm{x} \neq \bm{y}$). This property ensures that, for $\nu<1$, the coexistence of $\mu$-strong convexity and H\"{o}lder smoothness forces the domain to have a bounded diameter of at most $(L/\mu)^{\frac{1}{1-\nu}}$. In particular, no function defined on all of $\mathbb{R}^d$ can satisfy both properties with finite $L$. We use this fact in the proof of Theorem \ref{thm:sgdm_s} (see Lemma \ref{lem:sgdm_s_03}). When $\nu=1$, this reduces to the standard condition $\mu \leq L$.

\begin{thm}
[\textbf{Convergence Analysis of Vanilla SGD for Strongly Convex Functions}]\label{thm:sgd_s} Let the function $f \colon \mathbb{R}^d \to \mathbb{R}$ be $\mu$-strongly convex.
Suppose that Assumptions \ref{assum:01} and \ref{assum:02} hold. With learning rate $\eta_t := \frac{2\nu}{\mu t}$, if $\nu + 1 \leq \mathfrak{p}$ and $\eta_t^\nu \leq \frac{1}{L}$ for all $t \in [T]$, then:
\begin{align*}
\mathbb{E}\left[ f(\bm{x}_T) - f^\star\right] 
\leq \frac{E_1}{T^\nu} 
= \mathcal{O}\left( \frac{1}{T^\nu} \right).
\end{align*}
In particular, when $\nu + 1 = \mathfrak{p}$,
\begin{align*}
\mathbb{E}\left[ f(\bm{x}_T) - f^\star\right] 
= \mathcal{O}\left( \frac{1}{T^{\mathfrak{p}-1}} \right),
\end{align*}
where $E_1 := \max\left\{ f(\bm{x}_1) -f^\star, \frac{E_0L}{\nu(\nu+1)}\left(\frac{2\nu}{\mu} \right)^{\nu+1}\right\}$ and $E_0$ defined in Eq. \eqref{eq:E0} are non-negative constants.
\end{thm}

The proof of Theorem \ref{thm:sgd_s} is in Appendix \ref{sec:sgd_s}. Theorem~\ref{thm:sgd_s} shows that the convergence rate  deteriorates when $\nu$ or $(\mathfrak{p}-1)$ is sufficiently small. For $\nu=1$ and $\mathfrak{p}=2$, our result recovers the standard $\mathcal{O}(1/T)$ rate for vanilla SGD \citep{Nemirovski2009Rob, Rakhlin2012Mak}. Unlike \citet{Fatkhullin2025Can}, who require bounded gradients to achieve similar rates, our analysis holds without such restrictive assumptions.

\begin{thm}
[\textbf{Convergence Analysis of Vanilla SGD with Momentum for Strongly Convex Functions}]\label{thm:sgdm_s} Let the function $f: \mathbb{R}^d \to \mathbb{R}$ be $\mu$-strongly convex. Suppose that Assumptions \ref{assum:01} and \ref{assum:02} hold. With learning rate $\eta_t := \eta$, if $\nu + 1 \leq \mathfrak{p}$ and $\eta \leq \min \left\{ \frac{2(1-\beta)}{E_2}, \left(\frac{1}{8L}\right)^{\frac{1}{\nu}}, \left( \frac{(1-\beta)^{\nu+1}}{2E_2(1+2L)\beta^{\nu+1}}\right)^{\frac{1}{\nu+1}}, \left(\frac{(1-\beta)^{2\nu}}{2L^2\nu\beta^{2\nu}} \right)^{\frac{1}{2\nu}} \right\}$, then:
\begin{align*}
\mathbb{E}\left[ f(\bm{x}_T) - f^\star \right]
&\leq \frac{f(\bm{x}_1) - f^\star}{2} \left(1-\frac{\eta E_2}{4}\right)^T + \frac{E_3}{1-\beta} \\
&=\mathcal{O}\left( \left(1-\frac{\eta E_2}{4}\right)^T + \eta^{\nu+1}\right).
\end{align*}
In particular, when $\nu + 1 = \mathfrak{p}$,
\begin{align*}
\mathbb{E}\left[ f(\bm{x}_T) - f^\star \right]
&=\mathcal{O}\left( \left(1-\frac{\eta E_2}{4}\right)^T + \eta^{\mathfrak{p}}\right)
\end{align*}
where $E_2 := \frac{2\mu(\nu+1)}{\nu\left(1+L^{\frac{1}{\nu}}\right)+1}$, $E_3 := \frac{LE_0}{\nu+1} + \frac{(1+2L)E_0E_2}{4(\nu+1)}\left(\frac{\beta}{1-\beta} \right)^{\nu+1}\eta + \frac{2L^2\nu(1-\beta) E_0}{\nu+1}\left( \frac{\beta}{1-\beta}\right)^{2\nu}\eta^{\nu}$, and $E_0$ defined in Eq. \eqref{eq:E0} are non-negative constants.
\end{thm}

The proof of Theorem \ref{thm:sgdm_s} is in Appendix \ref{sec:sgdm_s}. Under standard assumptions ($\nu=1, \mathfrak{p}=2$), vanilla SGD with momentum is known to achieve an $\mathcal{O}((1-\mu \eta)^t + \eta)$ convergence rate for strongly convex functions \citep{Liu2020AnI}. Theorem \ref{thm:sgdm_s} recovers this classical result as a special case. To the best of our knowledge, Theorem \ref{thm:sgdm_s} provides the first convergence analysis in expectation for vanilla SGD with momentum under heavy-tailed noise.

\subsection{Convex Case}
For the convex setting, we introduce an additional assumption regarding the boundedness of the iterates and the auxiliary sequence $\bm{z}_t$, where $\bm{z}_t := \bm{x}_t - \frac{\eta\beta}{1-\beta}\bm{m}_{t-1}$ is defined for the analysis of SGD with momentum.
\begin{assum}\label{assum:03}
There exist constants $D_1, D_2 > 0$ such that, for all $t \in \mathbb{N}$:\\
(i) $\| \bm{x}_t -\bm{x}^\star \| \leq D_1$,\ \ \ (ii) $\| \bm{z}_t -\bm{x}^\star \| \leq D_2$.
\end{assum}
While Assumption~\ref{assum:03} is relatively strong, it is a standard way to derive performance upper bounds in both convex and nonconvex optimization \citep{Nemirovski2009Rob, Kingma2015Ada, Reddi2018Ont, Zhuang2020Ada}. Note that while we do not explicitly assume boundedness of the gradient even in the convex case, combining Assumption~\ref{assum:03}(i) and Assumption \ref{assum:01} inherently ensures that the gradient remains bounded.

\begin{thm}
[\textbf{Convergence Analysis of Vanilla SGD for Convex Functions}]\label{thm:sgd_c} Let the function $f: \mathbb{R}^d \to \mathbb{R}$ be convex. Suppose that Assumptions \ref{assum:01}, \ref{assum:02}, and \ref{assum:03}(i) hold. With learning rate $\eta_t := \eta_{\max}t^{-\frac{1}{\nu+1}}$, if $\nu + 1 \leq \mathfrak{p}$, then:
\begin{align*}
&\min_{1\leq t\leq T} \mathbb{E}\left[ f(\bm{x}_t) - f^\star\right] \\
&\quad\leq \frac{\| \bm{x}_1 - \bm{x}^\star\|^{\nu+1}}{(\nu+1)D_1^{\nu-1}}\frac{1}{\sum_{t=1}^{T}\eta_t} +E_4 \frac{\sum_{t=1}^{T}\eta_t^{\nu+1}}{\sum_{t=1}^{T}\eta_t}\\
&\quad=\mathcal{O}\left( \frac{\log T}{T^{\frac{\nu}{\nu+1}}}\right).
\end{align*}
In particular, when $\nu + 1 = \mathfrak{p}$,
\begin{align*}
&\min_{1\leq t\leq T} \mathbb{E}\left[ f(\bm{x}_t) - f^\star\right] 
=\mathcal{O}\left( \frac{\log T}{T^{\frac{\mathfrak{p}-1}{\mathfrak{p}}}}\right),
\end{align*}
where $E_4 := 2^{1-\nu} \left( E_0 + \frac{\nu+1}{2}L^2D_1^{2\nu} \right)$ and $E_0$ defined in Eq. \eqref{eq:E0} are non-negative constants.
\end{thm}
The proof of Theorem \ref{thm:sgd_c} is in Appendix \ref{sec:sgd_c}. Under conventional assumptions ($\nu=1, \mathfrak{p}=2$), vanilla SGD is known to achieve a convergence rate of $\mathcal{O}\big((\sum \eta_t)^{-1} + \sum \eta_t^{2} / \sum \eta_t\big)$ for convex functions \citep{Nemirovski2009Rob, Gower2021SGD}. Theorem \ref{thm:sgd_c} recovers this standard result as a special case. Under Assumptions \ref{assum:01} and \ref{assum:02}, \cite[Theorem 3.2]{Fatkhullin2025Can} previously derived a similar convergence bound, so the novelty of Theorem \ref{thm:sgd_c} is incremental.

\begin{thm}
[\textbf{Convergence Analysis of Vanilla SGD with Momentum for Convex Functions}]\label{thm:sgdm_c} Let the function $f: \mathbb{R}^d \to \mathbb{R}$ be convex. Suppose that Assumptions \ref{assum:01}, \ref{assum:02}, \ref{assum:03}(i), and \ref{assum:03}(ii) hold. With learning rate $\eta_t := \eta$, if $\nu + 1 \leq \mathfrak{p}$, then:
\begin{align*}
&\min_{1 \leq t \leq T} \mathbb{E}\left[f(\bm{x}_t) - f^\star \right] \\
&\quad\leq \frac{\| \bm{x}_1 - \bm{x}^\star\|^{\nu+1}}{\eta(\nu+1)D_2^{\nu-1}T} + E_6\eta + \frac{2^{1-\nu}E_5}{(\nu+1)D_2^{\nu-1}}\eta^\nu \\
&\quad=\mathcal{O}\left( \frac{1}{\eta T} + \eta^\nu\right).
\end{align*}
In particular, when $\nu + 1 = \mathfrak{p}$,
\begin{align*}
&\min_{1 \leq t \leq T} \mathbb{E}\left[ f(\bm{x}_t) - f^\star \right]
=\mathcal{O}\left( \frac{1}{\eta T} + \eta^{\mathfrak{p}-1}\right),
\end{align*}
where $E_5 := 2^{1-\nu}\left( \frac{C_\mathfrak{p} \sigma^\mathfrak{p}}{b^{\mathfrak{p}-1}}\right)^{\frac{\nu+1}{\mathfrak{p}}} + L^{\nu+1}D_1^{\nu(\nu+1)}$ and $E_6 := \frac{\beta\left(E_5 + \nu L^{\frac{\nu+1}{\nu}}D_1^{\nu+1} \right)}{(1-\beta)(\nu+1)}$ are non-negative constants.
\end{thm}
The proof of Theorem \ref{thm:sgdm_c} is in Appendix \ref{sec:sgdm_c}. Under standard assumptions ($\nu=1, \mathfrak{p}=2$), vanilla SGD with momentum is known to achieve a convergence rate of $\mathcal{O}(1/(\eta T) + \eta)$ for convex functions \citep{Sebbouh2021Alm}. Theorem \ref{thm:sgdm_c} recovers this conventional result as a special case. To the best of our knowledge, Theorem \ref{thm:sgdm_c} provides the first convergence analysis in expectation for vanilla SGD with momentum on convex functions under heavy-tailed noise.

\subsection{Nonconvex Case}
\begin{thm}
[\textbf{Convergence Analysis of Vanilla SGD for Nonconvex Functions}]\label{thm:sgd_n} Let the function $f: \mathbb{R}^d \to \mathbb{R}$ be nonconvex. Suppose that Assumptions \ref{assum:01} and \ref{assum:02} hold. With learning rate $\eta_t := \eta_{\max}t^{-\frac{1}{\nu+1}}$, if $\nu + 1 \leq \mathfrak{p}$ and $\eta_t^\nu < \frac{2}{L}$ for all $t \in [T]$, then:
\begin{align*}
&\min_{1 \leq t \leq T}\mathbb{E}\left[\| \nabla f(\bm{x}_t) \|^2 \right] \\
&\quad\leq \frac{ 2(f(\bm{x}_1) - f^\star)}{\left( 2- L\eta_{\max}^{\nu}\right)} \frac{1}{\sum_{t=1}^{T} \eta_t} + E_7 \frac{\sum_{t=1}^{T} \eta_t^{\nu+1}}{ \sum_{t=1}^{T} \eta_t} \\
&\quad= \mathcal{O}\left(\frac{\log T}{T^{\frac{\nu}{\nu+1}}} \right).
\end{align*}
In particular, when $\nu + 1 = \mathfrak{p}$,
\begin{align*}
&\min_{1 \leq t \leq T}\mathbb{E}\left[\| \nabla f(\bm{x}_t) \|^2 \right] 
= \mathcal{O}\left(\frac{\log T}{T^{\frac{\mathfrak{p}-1}{\mathfrak{p}}}} \right),
\end{align*}
where $E_7 := \frac{2E_0L}{(\nu+1)\left( 2-L\eta_{\max}^\nu\right)}$ and $E_0$ defined in Eq. \eqref{eq:E0} are non-negative constants.
\end{thm}
The proof of Theorem \ref{thm:sgd_n} is in Appendix \ref{sec:sgd_n}. Under standard assumptions ($\nu=1, \mathfrak{p}=2$), vanilla SGD is known to achieve a convergence rate of $\mathcal{O}\big((\sum \eta_t)^{-1} + \sum \eta_t^{2} / \sum \eta_t\big)$ for nonconvex functions \citep{Sebbouh2021Alm, Khaled2023Bet}. Theorem \ref{thm:sgd_n} recovers this classical result as a special case. While \citet[Theorem 5.4]{Fatkhullin2025Can} previously derived a similar bound under Assumptions~\ref{assum:01} and \ref{assum:02}, their analysis necessitates assuming bounded gradients. In contrast, our framework eliminates this requirement, providing a more general and robust convergence guarantee for the nonconvex setting.

\begin{thm}
[\textbf{Convergence Analysis of Vanilla SGD with Momentum for Nonconvex Functions}]\label{thm:sgdm_n} Let the function $f: \mathbb{R}^d \to \mathbb{R}$ be nonconvex. Suppose that Assumptions \ref{assum:01} and \ref{assum:02} hold. With learning rate $\eta_t := \eta$, if $\nu + 1 \leq \mathfrak{p}$ and $\eta \leq \min\left\{\left(\frac{1}{4L}\right)^{\frac{1}{\nu}}, \frac{1-\beta}{\beta}\left(\frac{1}{4\nu L^2}\right)^{\frac{1}{2\nu}}\right\}$, then:
\begin{align*}
&\min_{1 \leq t \leq T}\mathbb{E}\left[\| \nabla f(\bm{x}_t) \|^2 \right] \\
&\quad\leq \frac{4(f(\bm{x}_1) - f^\star)}{\eta T} + \frac{LE_0\eta^{\nu}}{4(\nu+1)} + E_8\eta^{2\nu}\\
&\quad= \mathcal{O}\left( \frac{1}{\eta T} + \eta^{\nu}\right).
\end{align*}
In particular, when $\nu + 1 = \mathfrak{p}$,
\begin{align*}
&\min_{1 \leq t \leq T}\mathbb{E}\left[\| \nabla f(\bm{x}_t) \|^2 \right] 
= \mathcal{O}\left( \frac{1}{\eta T} + \eta^{\mathfrak{p}-1}\right),
\end{align*}
where $E_8 := \frac{2L}{\nu+1}\left(\frac{\beta}{1-\beta} \right)^{2\nu}\left\{L(1-\nu)+2E_0\nu \right\}$ and $E_0$ defined in Eq. \eqref{eq:E0} are non-negative constants.
\end{thm}
The proof of Theorem \ref{thm:sgdm_n} is in Appendix \ref{sec:sgdm_n}. Under standard assumptions ($\nu=1, \mathfrak{p}=2$), vanilla SGD with momentum is known to achieve a convergence rate of $\mathcal{O}(1/(\eta T) + \eta)$ for nonconvex functions \citep{Yan2018Uni, Mai2020Con, Liu2020AnI}. Our analytical framework builds upon the approach established by \citet{Liu2020AnI}, and Theorem \ref{thm:sgdm_n} recovers their results as a special case when $\nu=1$ or $\mathfrak{p}=2$. To the best of our knowledge, Theorem \ref{thm:sgdm_n} represents the first convergence analysis in expectation for vanilla SGD with momentum on nonconvex functions under heavy-tailed noise.

\paragraph{Comparison with \citep{Liu2025Non}} \citet{Liu2025Non} established that normalized SGD with momentum achieves the following convergence rates for nonconvex functions under heavy-tailed noise:
\begin{align*}
&\min_{1 \leq t \leq T} \mathbb{E}\left[ \| \nabla f(\bm{x}_t) \|\right] = \mathcal{O}\left( T^{-\frac{\mathfrak{p}-1}{3\mathfrak{p}-2}} \right) \ \text{(with known $\mathfrak{p}$)}, \\
&\min_{1\leq t \leq T} \mathbb{E}\left[ \| \nabla f(\bm{x}_t) \| \right] = \mathcal{O}\left( T^{-\frac{\mathfrak{p}-1}{2\mathfrak{p}}} \right) \ \text{(with unknown $\mathfrak{p}$)}.
\end{align*}
In contrast, our Theorem \ref{thm:sgdm_n} implies that vanilla SGD with momentum achieves:
\begin{align*}
&\min_{1\leq t \leq T} \mathbb{E}\left[ \| \nabla f(\bm{x}_t) \| \right] = \mathcal{O}\left(T^{-\frac{\mathfrak{p}-1}{2\mathfrak{p}}} \right) \ \text{(with $\eta = T^{-1/\mathfrak{p}}$)}, \\
&\min_{1\leq t \leq T} \mathbb{E}\left[ \| \nabla f(\bm{x}_t) \| \right] = \mathcal{O}\left( T^{-\frac{p-1}{4}} \right) \ \text{(with $\eta = T^{-1/2}$)}.
\end{align*}
Since $\frac{\mathfrak{p}-1}{3\mathfrak{p}-2} \geq \frac{\mathfrak{p}-1}{2\mathfrak{p}} \geq \frac{\mathfrak{p}-1}{4}$ for $\mathfrak{p} \in (1,2]$, these results illustrate the efficiency of normalization in mitigating heavy-tailed noise, especially considering that the rate $\mathcal{O}\left( T^{-\frac{\mathfrak{p}-1}{3\mathfrak{p}-2}} \right)$ is optimal \citep{Zhang2020Why}. While vanilla SGD with momentum exhibits sub-optimal rates compared to its normalized counterparts, our work fills a critical gap by characterizing the fundamental convergence behavior of the most widely used momentum-based optimizer without explicit gradient control. Our findings reveal that vanilla SGD with momentum maintains sublinear convergence even without normalization or clipping, providing a necessary theoretical baseline that more complex algorithms should aim to surpass.

\paragraph{Analytical framework} As demonstrated, our analysis extends classical results established under $L$-smoothness and bounded variance to more general settings. These results are obtained by integrating existing analytical frameworks with the technical tools developed in Section \ref{sec:tool}. Specifically, for vanilla SGD with momentum, we employ an auxiliary sequence and a Lyapunov function to facilitate the analysis, following established methodologies (e.g., \citep{Ghadimi2015Glo, Gadat2018Sto, Mai2020Con, Liu2020AnI, Defazio2021Mom, Qiu2024Con, Kondo2025Acc}). Furthermore, for convex and nonconvex settings, we adopt the weighting scheme introduced by \citet{Stich2019Uni} to streamline the convergence proofs. Notably, our framework is sufficiently general to potentially establish convergence guarantees under heavy-tailed noise for other popular optimizers, such as Adam \citep{Kingma2015Ada} and Muon \citep{Jordan2024Muo}.

\paragraph{Summary of theoretical findings} Our theoretical analysis provides three key insights. First, vanilla SGD with momentum guarantees convergence in expectation without requiring explicit gradient control, such as clipping or normalization. Second, while convergence is guaranteed, the rate is inherently slower than those achieved under classical $L$-smoothness and bounded variance, or those obtained by heavy-tailed-specific methods (e.g., clipping-based algorithms). Third, the condition $\nu + 1 \leq \mathfrak{p}$ is essential for ensuring these convergence guarantees. In the following section, we provide numerical experiments to empirically validate these findings, with a particular focus on the necessity of this third condition.

\section{Numerical Experiments}
\label{sec:exp}
This section presents numerical experiments to empirically validate the theoretical findings established in Section \ref{sec:main}. We evaluate the performance of vanilla SGD and SGD with momentum on a series of synthetic strongly convex, convex, and nonconvex objective functions to observe and compare their convergence behaviors under heavy-tailed noise.

\begin{figure*}[!t]
\centering
\includegraphics[width=0.95\linewidth]{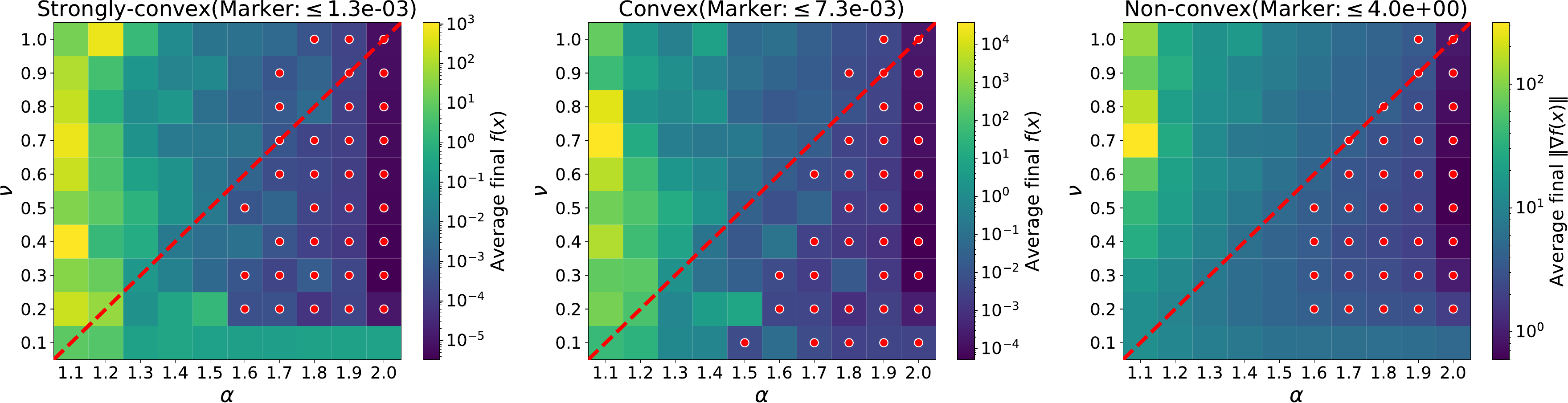}%
\caption{Convergence performance across various $(\nu, \alpha)$ pairs. The heatmaps show the average final objective values (for strongly convex and convex cases) and gradient norms (for nonconvex cases) after 4,000 steps of vanilla SGD with momentum. The top 35 performing combinations are indicated by red markers. Below the red dotted line, the theoretical condition $\nu+1 \leq \alpha$ is satisfied, ensuring convergence guarantees.}
\label{fig:heatmap_m}
\end{figure*}

\paragraph{Problem setup} We prepared objective functions defined as follows:
\begin{align}\label{eq:function}
&f_{\text{strongly convex}}(\bm{x}) = \frac{1}{2} \sum_{i=1}^d x_i^2 + \sum_{i=1}^d \frac{1}{\nu+1} |x_i|^{\nu+1},\nonumber \\
&f_{\text{convex}}(\bm{x}) = \sum_{i=1}^d \frac{1}{\nu+1} |x_i|^{\nu+1},\\
&f_{\text{nonconvex}}(\bm{x}) = \sum_{i=1}^d \left( \frac{1}{\nu+1} |x_i|^{\nu+1} - \frac{1}{2} \cos(2\pi x_i) + \frac{1}{2} \right) \nonumber,
\end{align}
where $\bm{x} := (x_1, x_2, \ldots, x_d)^\top \in \mathbb{R}^d$. For any $\nu \in (0,1]$, each of these functions satisfies Assumption \ref{assum:01} with a H\"{o}lder continuous gradient. In all experiments, the dimension is set to $d = 500$. The learning rate was set to $\eta_t := t^{-\frac{1}{\nu+1}}$ for convex and nonconvex cases, and $\eta_t := \frac{2\nu}{t}$ for strongly convex cases.

\paragraph{Heavy-tailed noise setup} The minibatch stochastic gradient $\nabla f_{\mathcal{S}_t}(\bm{x}_t)$ in Algorithm \ref{alg:1} was simulated by injecting synthetic heavy-tailed noise $\bm{\omega}_t \in \mathbb{R}^d$ into the full gradient: $\nabla f_{\mathcal{S}_t}(\bm{x}_t) = \nabla f(\bm{x}_t) + \lambda\bm{\omega}_t$, where $\lambda > 0$ scales the noise intensity. Following \citet{Fatkhullin2025Can}, each component of the synthetic noise was defined as:
\begin{align*}
(\bm{\omega}_t)_i := s_i u_i^{-\frac{1}{\alpha}},
\end{align*}
where $s_i \sim \text{Unif}(\{-1,1\}), u_i \sim \text{Unif}(0,1)$, and $\alpha \in (1,2)$ is the tail index. In this two-sided Pareto distribution, the $\mathfrak{p}$-th moment is finite for all $\mathfrak{p} \in (1, \alpha)$ and infinite for $\mathfrak{p} \geq \alpha$. Consequently, Assumption \ref{assum:02}(ii) is satisfied for any $\mathfrak{p} < \alpha$. Furthermore, as $\mathbb{E}_{s_i}[s_i]=0$, it follows from the independence of $s_i$ and $u_i$ that $\mathbb{E}_{(s_i, u_i)}[(\omega_t)_i]=\mathbb{E}_{s_i}[s_i]\mathbb{E}_{u_i}[u_i^{-1/\alpha}]=0$, ensuring that Assumption \ref{assum:02}(i) holds for any $\alpha \in (1,2)$. For the boundary case $\alpha=2$, we set $\bm{\omega}_t$ to standard Gaussian noise. The noise scale $\lambda$ was set to $0.01$ for strongly convex and convex functions, and $0.1$ for the nonconvex setting.

\begin{figure*}[!t]
  \centering
  \includegraphics[width=0.9\linewidth]{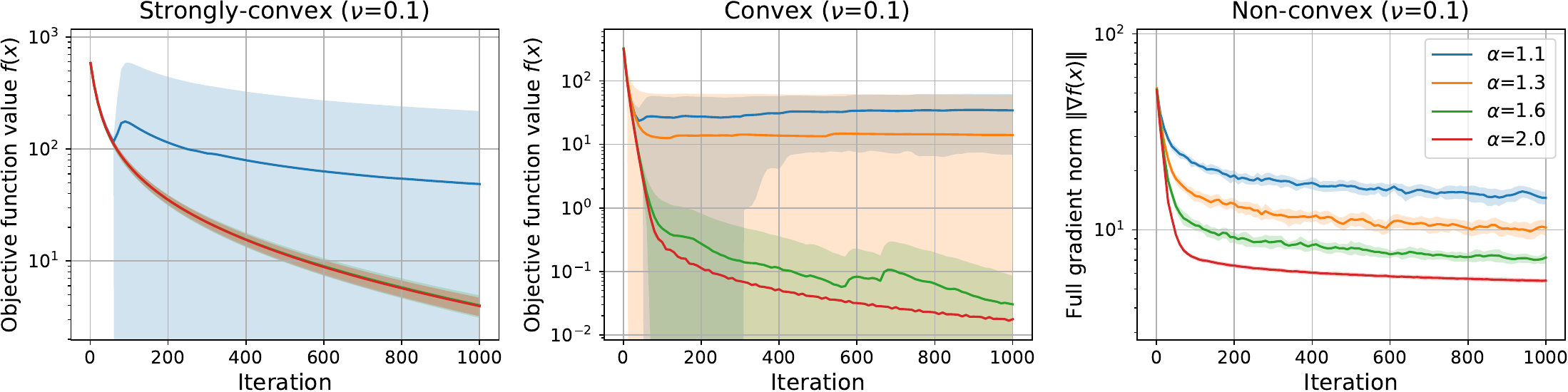}
  \includegraphics[width=0.9\linewidth]{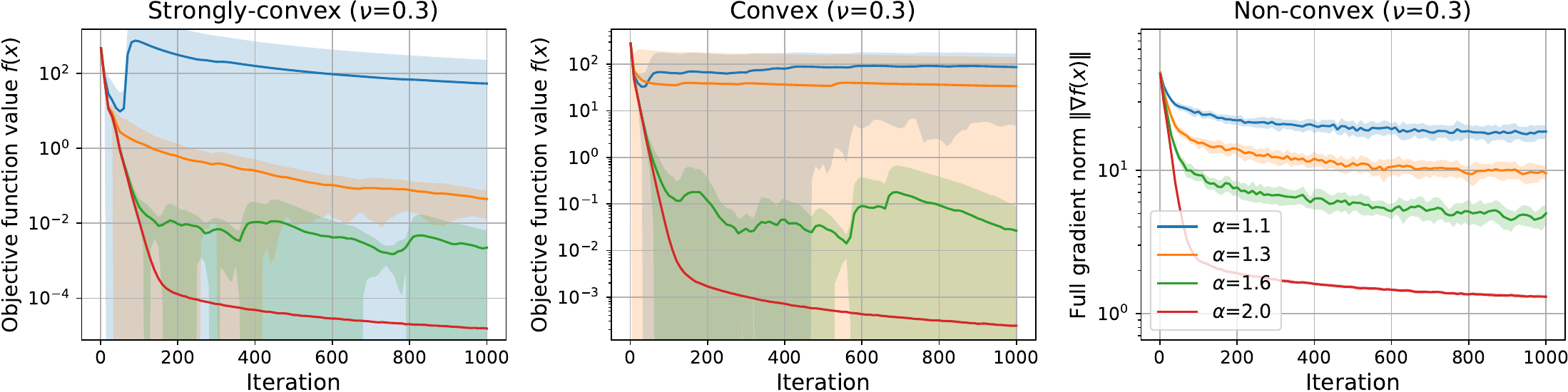}
  \includegraphics[width=0.9\linewidth]{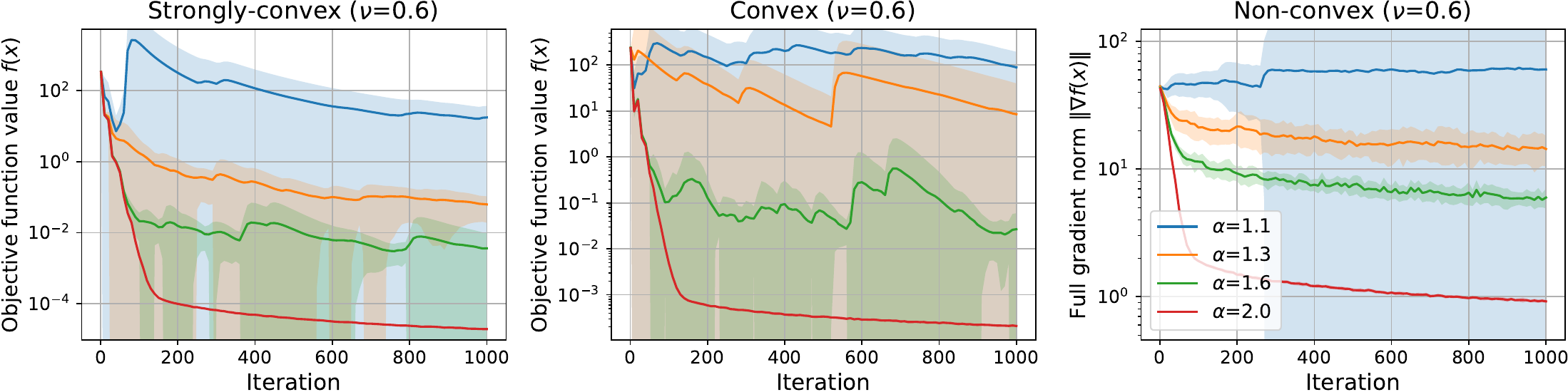}
  \includegraphics[width=0.9\linewidth]{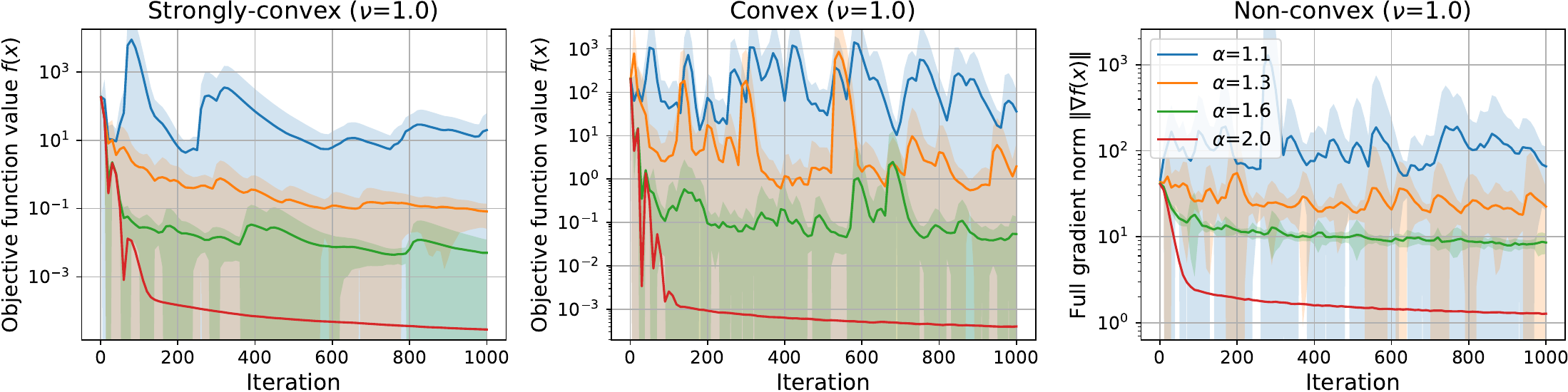}
\caption{Convergence trajectories of vanilla SGD with momentum. The plots illustrate the evolution of objective values (for strongly convex and convex cases) and gradient norms (for nonconvex cases) over 1,000 steps for the functions defined in Eq. (\ref{eq:function}). Solid lines represent the average of 20 independent trials, while the lightly shaded regions indicate the range between the minimum and maximum values.}
\label{fig:loss_m}
\end{figure*}

\paragraph{Dependence of convergence on $\nu$ and $\alpha$} All theorems in Section \ref{sec:main} necessitate the condition $\nu+1 \leq \alpha$ in order to leverage the moment bounds in Lemma \ref{lem:exp}. Consequently, our theory does not guarantee the convergence of vanilla SGD or SGD with momentum when this condition is violated. To empirically investigate the interplay between $\nu$ and $\alpha$, we conducted experiments across 100 combinations of $\nu \in \{0.1, \dots, 1.0\}$ and $\alpha \in \{1.1, \dots, 2.0\}$. For each pair, SGD with momentum was executed for 4,000 steps, and we recorded the final objective values (for strongly convex and convex) or gradient norms (for nonconvex). Figure~\ref{fig:heatmap_m} presents heatmaps averaged over 20 independent trials, where the top 35 performing combinations are highlighted with red markers. Notably, the condition $\nu+1 \leq \alpha$ corresponds to the lower-right triangular region of each map. As shown in the figure, the high-performance markers are predominantly concentrated within this region, empirically validating that $\nu+1 \leq \alpha$ is a near-necessary requirement for the stable convergence of vanilla SGD with momentum. Furthermore, we observe that for small values of $\nu$ or $\alpha$, the optimization performance remains poor even when the convergence condition is satisfied. This observation aligns with our theoretical rates; for instance, the nonconvex convergence rate $\mathcal{O}\big(1/{(\eta T}) + \eta^\nu\big)$ deteriorates when $\nu$ or $(\alpha - 1)$ is sufficiently small. Thus, these numerical results provide strong empirical support for our theoretical results.

\paragraph{Detailed observation of convergence behavior} Figure \ref{fig:loss_m} illustrates the evolution of objective values and gradient norms for fixed $\nu$ and varying $\alpha$. Solid lines represent the average of 20 independent trials, with shaded regions indicating the range between extrema. In all cases, a smaller $\alpha$ not only degrades average performance but also amplifies variance, visually confirming our theoretical insight: the divergence of lower-order moments manifests as extreme instability in optimization trajectories. Crucially, under highly heavy-tailed noise (small $\alpha$), larger $\nu$ values lead to more erratic loss curves, whereas smaller $\nu$ maintains relative stability. This aligns with the requirement $\nu+1 \leq \alpha$, as a smaller $\nu$, representing broader H\"{o}lder smoothness, widens the guaranteed convergence range of $\alpha$. These results underscore the inherent incompatibility between extreme heavy-tailed noise ($\alpha < 2$) and classical $L$-smoothness ($\nu=1$), justifying the adoption of more flexible smoothness classes, such as H\"{o}lder or $(L_0, L_1)$-smoothness, in stochastic optimization theory. (See Appendix \ref{sec:addexp} for corresponding results for vanilla SGD).

\paragraph{Experiments on WikiText-2 dataset with Transformer-XL}
\begin{figure}[!t]
\centering
\includegraphics[width=1\linewidth]{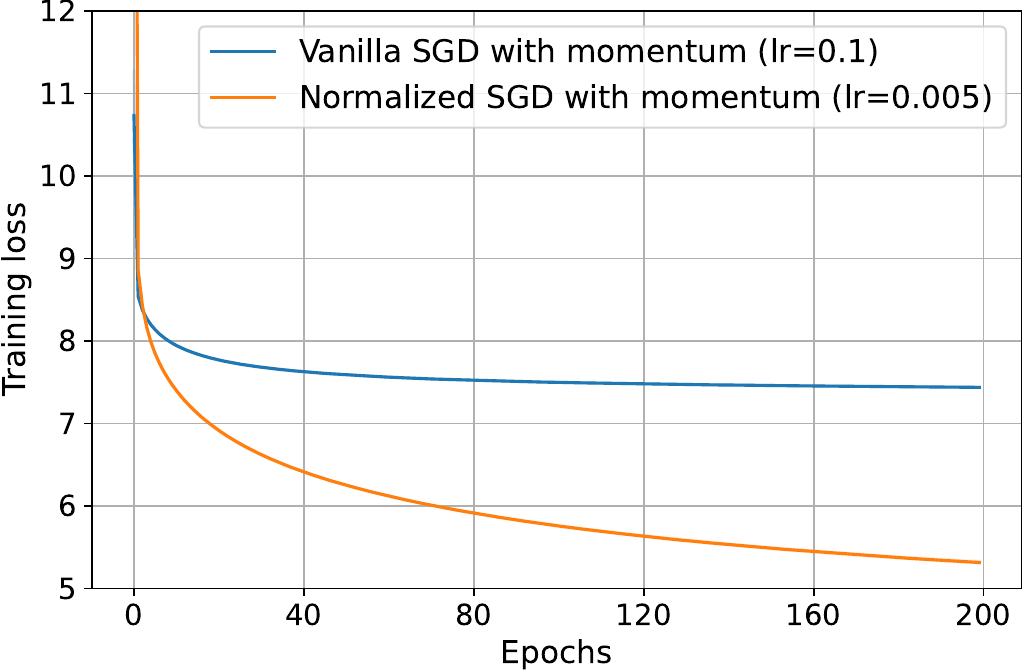}%
\caption{Loss function value for training versus the number of epochs in training Transformer-XL on the WikiText-2 dataset. The solid lines represent the mean value, and the shaded areas represent the maximum and minimum values over three runs.}
\label{fig:wt2}
\end{figure}
Finally, we conducted more realistic experiments on the WikiText-2 dataset \citep{Merity2017Int} with Transformer-XL, where heavy-tailed gradient noise has been observed \citep{Kunstner2023Noi}. We evaluated a 6-layer, 8-head Transformer-XL model with approximately 26.9M parameters. The embedding, head, and inner dimensions were set to 256, 32, and 1024, respectively, with tied input/output embeddings. Training was performed for 200 epochs with a batch size of 20, and a target sequence and memory length of 128. We compared vanilla SGD with momentum and normalized SGD with momentum under constant learning rates. For each optimizer, we searched learning rates over $\{0.001, 0.005, 0.01, 0.05, 0.1\}$ and selected the best-performing one on the basis of the final training loss. The best learning rates were 0.1 for vanilla SGD with momentum and 0.005 for normalized SGD with momentum. The momentum parameter was fixed to 0.9 for both optimizers. Figure \ref{fig:wt2} confirms that (1) vanilla SGD with momentum converges without gradient clipping or normalization and (2) normalized SGD with momentum achieves a lower loss faster, consistent with our theory. We note that validating $\nu+1 \leq \mathfrak{p}$ on real tasks would not be feasible since $\nu$ cannot be reliably estimated in practice. This is why the verification of this condition was conducted on synthetic functions where both $\nu$ and $\mathfrak{p}$ can be fully controlled.

\section{Conclusion}
In this study, we established a comprehensive convergence analysis for vanilla SGD and SGD with momentum under heavy-tailed noise, encompassing strongly convex, convex, and nonconvex functions. A key departure from traditional theory is our shift from classical $L$-smoothness and bounded variance to a more general framework of H\"{o}lder continuous gradients. This approach naturally extends existing results while proving that convergence in expectation is guaranteed for SGD with momentum, even without explicit gradient control such as clipping or normalization. Our framework identifies the critical condition $\nu + 1 \leq \mathfrak{p}$, capturing the intrinsic interplay between function smoothness and the noise tail index. Numerical experiments empirically validate this condition, demonstrating that H\"{o}lder continuity offers a more robust and realistic foundation for heavy-tailed optimization than standard $L$-smoothness.

\begin{acknowledgements} 
We are sincerely grateful to Program Chairs, Area Chairs, and four anonymous reviewers for helping us improve the original manuscript. This work was supported by the Japan Society for the Promotion of Science (JSPS) KAKENHI Grants, Numbers 24K14846 and 26KJ2026, awarded to Hideaki Iiduka and Naoki Sato, respectively. 
\end{acknowledgements}

\bibliography{uai2026}

\newpage
\appendix
\onecolumn

\newtheorem*{manuallem1}{Lemma \ref{lem:holder} (restated)}
\newtheorem*{manuallem2}{Lemma \ref{lem:q-norm} (restated)}
\newtheorem*{manuallem3}{Lemma \ref{lem:moment} (restated)}
\newtheorem*{manuallem4}{Lemma \ref{lem:exp} (restated)}
\newtheorem*{manuallem5}{Lemma \ref{lem:aux} (restated)}
\newtheorem*{manualthm_sgd_s}{Theorem \ref{thm:sgd_s} (restated)}
\newtheorem*{manualthm_sgdm_s}{Theorem \ref{thm:sgdm_s} (restated)}
\newtheorem*{manualthm_sgd_c}{Theorem \ref{thm:sgd_c} (restated)}
\newtheorem*{manualthm_sgdm_c}{Theorem \ref{thm:sgdm_c} (restated)}
\newtheorem*{manualthm_sgd_n}{Theorem \ref{thm:sgd_n} (restated)}
\newtheorem*{manualthm_sgdm_n}{Theorem \ref{thm:sgdm_n} (restated)}

\title{Vanilla SGD with Momentum Survives Heavy-Tailed Noise: \\Convergence Analysis without Gradient Clipping or Normalization\\(Supplementary Material)}
\maketitle

\addtocontents{toc}{\protect\setcounter{tocdepth}{2}}
\tableofcontents

\paragraph{Notation.}
Throughout the appendix, $t$ and $T$ retain the same meaning as in Algorithm \ref{alg:1}: $T \in \mathbb{N}$ denotes the total number of iterations and $t$ indexes the corresponding iterates $(\bm{x}_t)_{t \ge 1}$ and $(\bm{m}_t)_{t \ge 0}$ generated by Algorithm \ref{alg:1}.

\section{Proofs of Lemmas in Section \ref{sec:pre}}
\label{sec:proof_A}
\subsection{Proof of Lemma \ref{lem:holder}}
\begin{manuallem1}
Suppose that Assumption \ref{assum:01} holds. Then, for all $\bm{x}, \bm{y} \in \mathbb{R}^d$, the following inequality holds.
\begin{align*}
f(\bm{y}) \leq f(\bm{x}) + \langle \nabla f(\bm{x}), \bm{y}-\bm{x} \rangle + \frac{L}{\nu + 1} \| \bm{y} - \bm{x} \|^{\nu+1}.
\end{align*}
\end{manuallem1}
\begin{proof}
For all $t \geq 0$, let us define the function $h\colon \mathbb{R} \to \mathbb{R}$ as $h(t) := f(t\bm{y} + (1-t)\bm{x}) = f(\bm{x} + t(\bm{y}-\bm{x}))$. Then, we have 
\begin{align*}
h'(t) = \langle \nabla f(\bm{x}+t(\bm{y}-\bm{x})), \bm{y}-\bm{x}\rangle.
\end{align*} 
From the fundamental theorem of calculus,
\begin{align*}
f(\bm{y}) - f(\bm{x})
= h(1) - h(0)
= \int_0^1 h'(t) d_t
= \int_0^1 \langle \nabla f(\bm{x}+t(\bm{y}-\bm{x})), \bm{y}-\bm{x}\rangle dt,
\end{align*}
which implies
\begin{align*}
f(\bm{y}) - f(\bm{x}) 
&= \int_{0}^{1} \langle \nabla f(\bm{x} + t(\bm{y} - \bm{x})), \bm{y} - \bm{x} \rangle dt \\
&= \langle \nabla f(\bm{x}), \bm{y} - \bm{x} \rangle +  \int_{0}^{1} \langle \nabla f(\bm{x} + t(\bm{y} - \bm{x})) - \nabla f(\bm{x}), \bm{y} - \bm{x} \rangle dt.
\end{align*}
From the Cauchy-Schwarz inequality and Assumption \ref{assum:01}, 
\begin{align*}
f(\bm{y}) - f(\bm{x}) 
&\leq \langle \nabla f(\bm{x}), \bm{y} - \bm{x} \rangle +  \int_{0}^{1} \| \nabla f(\bm{x} + t(\bm{y} - \bm{x})) - \nabla f(\bm{x})\| \| \bm{y} - \bm{x} \| dt \\
&\leq \langle \nabla f(\bm{x}), \bm{y} - \bm{x} \rangle + L \| \bm{y} - \bm{x} \|^{\nu+1}\int_{0}^{1} t^\nu dt.
\end{align*}
Here, we have 
\begin{align*}
\int_{0}^{1} t^\nu dt = \left[ \frac{t^{\nu+1}}{\nu+1} \right]_{0}^{1} 
= \frac{1}{\nu+1}.
\end{align*}
Therefore, we find that
\begin{align*}
f(\bm{y}) - f(\bm{x}) 
\leq \langle \nabla f(\bm{x}), \bm{y} - \bm{x} \rangle + \frac{L}{\nu+1} \| \bm{y} - \bm{x} \|^{\nu+1}.
\end{align*}
This completes the proof.
\end{proof}

\subsection{Proof of Lemma \ref{lem:q-norm}}
\begin{manuallem2}
Let $q \in (1,2]$. For any $\bm{y} \in \mathbb{R}^d$ and any $\bm{x} \in \mathbb{R}^d \setminus \{\bm{0}\}$, the following inequality holds:
\begin{align*}
\| \bm{x} \pm \bm{y} \|^q
\leq \| \bm{x} \|^q \pm q\| \bm{x} \|^{q-2}\langle \bm{x}, \bm{y} \rangle + 2^{2-q}\| \bm{y} \|^q.
\end{align*}
\end{manuallem2}
\begin{proof}
The proof is trivial for $q = 2$. Thus, consider the case of $q \in (1,2)$. Let us define 
\begin{align*}
F(\bm{y}) := \| \bm{x} \|^q + q\| \bm{x} \|^{q-2}\langle \bm{x}, \bm{y}\rangle+2^{2-q}\| \bm{y} \|^q - \| \bm{x} + \bm{y} \|^q.
\end{align*}
Accordingly, we have 
\begin{align*}
\nabla F(\bm{y}) 
= q\| \bm{x} \|^{q-2}\bm{x} + 2^{2-q}q\| \bm{y} \|^{q-2}\bm{y} - q\| \bm{x}+\bm{y} \|^{q-2}(\bm{x}+\bm{y}).
\end{align*}
When $\nabla F(\bm{y}) = \bm{0}$, the following holds:
\begin{align*}
\left( \|\bm{x}\|^{q-2}-\| \bm{x}+\bm{y}\|^{q-2}\right)\bm{x} + \left( 2^{2-q}\|\bm{y} \|^{q-2} - \| \bm{x}+\bm{y} \|^{q-2}\right)\bm{y} = \bm{0}.
\end{align*}
Therefore, when $\nabla F(\bm{y}) = \bm{0}$, $\bm{x}$ and $\bm{y}$ are parallel. That is, it suffices to show that $F(\bm{y}) \geq 0$ holds when $\bm{x}$ and $\bm{y}$ are parallel. Then, let $\bm{y} = t\bm{x}$ $(t \in \mathbb{R})$. In this case, the inequality to be shown is:
\begin{align*}
|t+1|^q \leq 1 + qt + 2^{2-q}|t|^q.
\end{align*}
Here, let us define $g(t) := 1 + qt + 2^{2-q}|t|^q - |t+1|^q$. We will show that $g(t) \geq 0$ for all $t \in \mathbb{R}$. Note that $g(0) = 0$.

\paragraph{Case (i): $t \geq 0$ i.e., $|t|=t$ and $|t+1| = t+1$.} In this case, we have
\begin{align*}
g(t) = 1 + qt + 2^{2-q}t^q- (t+1)^q.
\end{align*}
Therefore, 
\begin{align*}
g'(t) = q\left\{ 1 + 2^{2-q}t^{q-1} - (t+1)^{q-1}\right\}.
\end{align*}
Since $f(x) = x^{q-1}$ is a concave function and satisfies $f(0) \geq 0$, and given that any such concave function satisfies $f(x+y) \leq f(x) + f(y)$, we have
\begin{align*}
g'(t) 
&\geq q\left\{ 1 + 2^{2-q}t^{q-1} - (t^{q-1}+1)\right\} \\
&= q\left(2^{2-q}-1\right)t^{q-1} \\
&>0,
\end{align*}
where we have used $2^{2-q} > 1$ in the last inequality. Therefore, $g(t)$ is monotone increasing for $t \geq 0$. Since $g(0) = 0$, it follows that $g(t) \geq 0$ for $t \geq 0$.

\paragraph{Case (ii): $-1 \leq t < 0$ i.e., $|t|=-t$ and $|t+1| = t+1$.} In this case, we have
\begin{align*}
g(t) = 1 + qt + 2^{2-q}(-t)^q-(t+1)^q.
\end{align*}
Therefore, 
\begin{align*}
g'(t) = q\left\{ 1 - 2^{2-q}(-t)^{q-1} - (t+1)^{q-1}\right\}.
\end{align*}
Here, let us define $u := -t$ $(u \in (0,1])$. Then,
\begin{align*}
g'(t) = q\{ 1 - \underbrace{(2^{2-q}u^{q-1} + (1-u)^{q-1})}_{=:\psi(u)}\}.
\end{align*}
Since $\psi(u)$ is the sum of concave functions, $\psi(u)$ itself is also concave. Therefore, the minimum point of $\psi(u)$ on an interval is an endpoint. From $\displaystyle\lim_{u\to0}\psi(u) = 1$ and $\psi(1) = 2^{2-q} > 1$, we have $\psi(u) > 1$ and $g'(t) < 0$. Therefore, $g(t)$ is monotone decreasing for $t \in [-1,0)$. Since $g(-1)= 1-q+2^{2-q}>0$ and $g(0)=0$, it follows that $g(t) \geq 0$ for $t \in [-1,0)$.

\paragraph{Case (iii): $t < -1$ i.e., $|t|=-t$ and $|t+1| = -(t+1)$.} In this case, we have
\begin{align*}
g(t) = 1 + qt + 2^{2-q}(-t)^q-(-t-1)^q.
\end{align*}
Therefore, 
\begin{align*}
g'(t) = q\left\{ 1 - 2^{2-q}(-t)^{q-1} + (-t-1)^{q-1}\right\}.
\end{align*}
Let us define $v := -t$ $(v > 1)$. Here, we have $(-t)^{q-1} = v^{q-1}$ and $(-t-1)^{q-1} = (v-1)^{q-1}$. Accordingly, 
\begin{align*}
g'(t) 
&= q\{ 1 - 2^{2-q}v^{q-1} + (v-1)^{q-1}\}.
\end{align*}
From $0<q-1<1$ and Jensen's inequality, 
\begin{align*}
\frac{1+(v-1)^{q-1}}{2}
\leq \left\{ \frac{1+(v-1)}{2}\right\}^{q-1}
= \frac{v^{q-1}}{2^{\,q-1}}.
\end{align*}
Then, we have $1 + (v-1)^{q-1} \le 2^{2-q}v^{q-1}$, which implies $g'(t)\le 0$.
Therefore, $g(t)$ is monotone decreasing for $t < -1$. Since $g(-1)=1-q+2^{2-q}>0$, it follows that $g(t) \geq 0$ for $t < -1$.

Therefore, $g(t) \geq 0$ holds for any $t \in \mathbb{R}$, and the proof is complete.
\end{proof}

\subsection{Proof of Lemma \ref{lem:moment}}
\begin{lem}[Theorem 3.1 in \citep{Pinelis2022Mul}]\label{lem:vBE_vector}
Let $\mathfrak{p} \in (1,2]$, and let $\bm{X}_1, \bm{X}_2, \ldots, \bm{X}_n \in \mathbb{R}^d$ be a sequence of independent random variables with $\mathbb{E}_{{\bm{X}}_i}\left[\bm{X}_i\right] = \bm{0}$ and $\mathbb{E}_{{\bm{X}}_i}\left[\| \bm{X}_i \|^\mathfrak{p} \right] < \infty$ for all $i \in [n]$. Then,
\begin{align*}
\mathbb{E}_{({{\bm{X}}_1},...,{{\bm{X}}_n})}\left[ \left\| \sum_{i \in [n]} \bm{X}_i\right\|^\mathfrak{p}\right]
\leq C_{\mathfrak{p}}\sum_{i \in [n]} \mathbb{E}_{{\bm{X}}_i}\left[ \| \bm{X}_i \|^\mathfrak{p} \right],
\end{align*}
where $C_\mathfrak{p}$ is a constant that depends only on $\mathfrak{p}$ and continuously and strictly decreasing in $\mathfrak{p} \in (1,2]$ from $\displaystyle\lim_{\mathfrak{p} \to 1} C_{\mathfrak{p}} = 2$ to $C_2 = 1$.
\end{lem}

\begin{manuallem3}
Suppose that Assumption \ref{assum:02} holds. Then, for all $\bm{x} \in \mathbb{R}^d$ that does not depend on $\bm{\xi}_t$, the following inequality holds:
\begin{align*}
\mathbb{E}_{{\bm{\xi}}_t}[\|\nabla f_{\mathcal{S}_t}(\bm{x}) - \nabla f(\bm{x})\|^{\mathfrak{p}}\big] \le \frac{C_{\mathfrak{p}}\sigma^{\mathfrak{p}}}{b^{\mathfrak{p}-1}},
\end{align*}
where $C_\mathfrak{p} \in [1,2)$ is the constant given in Lemma \ref{lem:vBE_vector}.
\end{manuallem3}
\begin{proof}
From Assumption \ref{assum:02} and Lemma \ref{lem:vBE_vector}, we have
\begin{align*}
\mathbb{E}_{\bm{\xi}_t}\left[\left\|\nabla f_{S_t}(\bm{x})-\nabla f(\bm{x})\right\|^{\mathfrak p}\right]
&=
\mathbb{E}_{\bm{\xi}_t}\left[
\left\|
\frac{1}{b}\sum_{i\in [b]} \mathsf{G}_{\xi_{t,i}}(\bm{x})-\nabla f(\bm{x})
\right\|^{\mathfrak p}
\right] \\
&=
\mathbb{E}_{\bm{\xi}_t}\left[
\left\|
\frac{1}{b}\sum_{i\in [b]}
\bigl(\mathsf{G}_{\xi_{t,i}}(\bm{x})-\nabla f(\bm{x})\bigr)
\right\|^{\mathfrak p}
\right] \\
&=
\frac{1}{b^{\mathfrak p}}
\mathbb{E}_{\bm{\xi}_t}\left[
\left\|
\sum_{i\in [b]}
\bigl(\mathsf{G}_{\xi_{t,i}}(\bm{x})-\nabla f(\bm{x})\bigr)
\right\|^{\mathfrak p}
\right] \\
&\le
\frac{C_{\mathfrak p}}{b^{\mathfrak p}}
\sum_{i\in [b]}
\mathbb{E}_{\bm{\xi}_t}\left[
\left\|\mathsf{G}_{\xi_{t,i}}(\bm{x})-\nabla f(\bm{x})\right\|^{\mathfrak p}
\right] \\
&=
\frac{C_{\mathfrak p}\sigma^{\mathfrak p}}{b^{\mathfrak p-1}}.
\end{align*}
This completes the proof.
\end{proof}

\subsection{Proof of Lemma \ref{lem:exp}}
\begin{manuallem4}
Let $\mathfrak{p} \in (1,2]$ and $\nu \in (0,1]$, suppose that $\nu+1 \leq \mathfrak{p}$. Then, for all random vectors $\bm{X} \in \mathbb{R}^d$, the following inequality holds:

\begin{align*}
\mathbb{E}_{\bm{X}} [\|\bm{X}\|^{\nu+1}] \leq \left( \mathbb{E}_{\bm{X}} [\|\bm{X}\|^{\mathfrak{p}}] \right)^{\frac{\nu+1}{\mathfrak{p}}}.
\end{align*}
\end{manuallem4}

\begin{proof}
Consider the function $g(x) = x^k$ defined for $x \geq 0$, where $k = \frac{\mathfrak{p}}{\nu+1}$. Since we assumed $\nu+1 \leq \mathfrak{p}$, it follows that $k \geq 1$. $g(x)$ is a convex function. From Jensen's inequality $g(\mathbb{E}_Y[Y]) \leq \mathbb{E}_Y[g(Y)]$, we have
\begin{align*}
\left( \mathbb{E}_{\bm{X}}[\|\bm{X}\|^{\nu+1}] \right)^{\frac{\mathfrak{p}}{\nu+1}} \leq \mathbb{E}_{\bm{X}}\left[ \left( \|\bm{X}\|^{\nu+1} \right)^{\frac{\mathfrak{p}}{\nu+1}} \right]
= \mathbb{E}_{\bm{X}}\left[ \|\bm{X}\|^\mathfrak{p} \right],
\end{align*}
where we have chosen $Y = \|\bm{X}\|^{\nu+1}$. Therefore, we have
\begin{align*}
\mathbb{E}_{\bm{X}}[\|\bm{X}\|^{\nu+1}] 
\leq \left( \mathbb{E}_{\bm{X}}\left[ \| \bm{X} \|^\mathfrak{p} \right] \right)^{\frac{\nu+1}{\mathfrak{p}}}.
\end{align*}
This completes the proof.
\end{proof}

\section{Proofs of Lemmas and Theorems in Section \ref{sec:main}}
\subsection{Proof of Lemma \ref{lem:aux}}
\label{sec:aux}
\begin{manuallem5}
Suppose that Assumptions \ref{assum:01} and \ref{assum:02} hold. Then, for all $\bm{x} \in \mathbb{R}^d$ that does not depend on $\bm{\xi}_t$, the following holds:
\begin{align*}
\mathbb{E}_{\bm{\xi}_t}\left[ \| \nabla f_{\mathcal{S}_t}(\bm{x}) \|^{\nu+1} \right] 
&\leq 2^{1-\nu} \left( \frac{C_\mathfrak{p} \sigma^\mathfrak{p}}{b^{\mathfrak{p}-1}} \right)^{\frac{\nu+1}{\mathfrak{p}}} +  \| \nabla f(\bm{x}) \|^{\nu+1}  \\
&\leq 2^{1-\nu} \left( \frac{C_\mathfrak{p} \sigma^\mathfrak{p}}{b^{\mathfrak{p}-1}} \right)^{\frac{\nu+1}{\mathfrak{p}}} + \frac{\nu+1}{2}  \| \nabla f(\bm{x}) \|^2  + \frac{1-\nu}{2}.
\end{align*}
\end{manuallem5}

\begin{proof}
From Lemma \ref{lem:q-norm}, we have
\begin{align*}
\| \nabla f_{\mathcal{S}_t}(\bm{x}) \|^{\nu+1}
&= \| \nabla f_{\mathcal{S}_t}(\bm{x}) - \nabla f(\bm{x}) + \nabla f(\bm{x}) \|^{\nu+1} \\
&\leq \| \nabla f(\bm{x}) \|^{\nu+1} - (\nu+1) \| \nabla f(\bm{x}) \|^{\nu-1}\langle \nabla f_{\mathcal{S}_t}(\bm{x}) - \nabla f(\bm{x}), \nabla f(\bm{x}) \rangle \\
&\quad+ 2^{1-\nu}\| \nabla f_{\mathcal{S}_t}(\bm{x}) - \nabla f(\bm{x}) \|^{\nu+1}.
\end{align*}
By taking the expectation, 
\begin{align*}
\mathbb{E}_{\bm{\xi}_t}\left[ \| \nabla f_{\mathcal{S}_t}(\bm{x}) \|^{\nu+1} \right] 
\leq \mathbb{E}_{\bm{\xi}_t}\left[ \| \nabla f(\bm{x}) \|^{\nu+1}\right] + 2^{1-\nu} \mathbb{E}_{\bm{\xi}_t}\left[ \| \nabla f_{\mathcal{S}_t}(\bm{x}) - \nabla f(\bm{x}) \|^{\nu+1} \right].
\end{align*} 
Here, from Lemmas \ref{lem:exp} and \ref{lem:moment}, we obtain
\begin{align*}
\mathbb{E}_{\bm{\xi}_t}\left[ \| \nabla f_{\mathcal{S}_t}(\bm{x}) - \nabla f(\bm{x}) \|^{\nu+1} \right]
&\leq \left( \mathbb{E}_{\bm{\xi}_t}\left[ \| \nabla f_{\mathcal{S}_t}(\bm{x}) - \nabla f(\bm{x}) \|^\mathfrak{p} \right]\right)^{\frac{\nu+1}{\mathfrak{p}}} \\
&\leq \left( \frac{C_\mathfrak{p} \sigma^\mathfrak{p}}{b^{\mathfrak{p}-1}} \right)^{\frac{\nu+1}{\mathfrak{p}}}.
\end{align*}
Therefore, we have
\begin{align*}
\mathbb{E}_{\bm{\xi}_t}\left[ \| \nabla f_{\mathcal{S}_t}(\bm{x}) \|^{\nu+1} \right] 
\leq \mathbb{E}_{\bm{\xi}_t}\left[ \| \nabla f(\bm{x}) \|^{\nu+1}\right] + 2^{1-\nu} \left( \frac{C_\mathfrak{p} \sigma^\mathfrak{p}}{b^{\mathfrak{p}-1}} \right)^{\frac{\nu+1}{\mathfrak{p}}}.
\end{align*}
In addition, from Young's inequality, 
\begin{align*}
\| \nabla f(\bm{x}) \|^{\nu+1} \cdot 1
\leq \frac{\nu+1}{2}\| \nabla f(\bm{x}) \|^2 + \frac{1-\nu}{2}.
\end{align*}
Therefore, 
\begin{align*}
\mathbb{E}_{\bm{\xi}_t}\left[ \| \nabla f_{\mathcal{S}_t}(\bm{x}) \|^{\nu+1} \right]
\leq 2^{1-\nu} \left( \frac{C_\mathfrak{p} \sigma^\mathfrak{p}}{b^{\mathfrak{p}-1}} \right)^{\frac{\nu+1}{\mathfrak{p}}} + \frac{\nu+1}{2}  \| \nabla f(\bm{x}) \|^2  + \frac{1-\nu}{2}.
\end{align*}
This completes the proof.
\end{proof}

\subsection{Proof of Theorem \ref{thm:sgd_s}}
\label{sec:sgd_s}
\begin{manualthm_sgd_s}
Let the function $f \colon \mathbb{R}^d \to \mathbb{R}$ be $\mu$-strongly convex. Suppose that Assumptions \ref{assum:01} and \ref{assum:02} hold. With learning rate $\eta_t := \frac{2\nu}{\mu t}$, if $\nu+1 \leq \mathfrak{p}$ and $\eta_t^\nu \leq \frac{1}{L}$ for all $t \in [T]$, then:
\begin{align*}
\mathbb{E}\left[ f(\bm{x}_T) - f^\star\right] 
\leq \frac{E_1}{T^\nu} 
= \mathcal{O}\left( \frac{1}{T^\nu} \right).
\end{align*}
In particular, when $\nu + 1 = \mathfrak{p}$,
\begin{align*}
\mathbb{E}\left[ f(\bm{x}_T) - f^\star\right] 
= \mathcal{O}\left( \frac{1}{T^{\mathfrak{p}-1}} \right),
\end{align*}
where $E_1 := \max\left\{ f(\bm{x}_1) -f^\star, \frac{E_0L}{\nu(\nu+1)}\left(\frac{2\nu}{\mu} \right)^{\nu+1}\right\}$ and $E_0$ defined in Eq. \eqref{eq:E0} are non-negative constants.
\end{manualthm_sgd_s}

\begin{proof}
From Lemma \ref{lem:holder}, we have
\begin{align*}
f(\bm{x}_{t+1})
&\leq f(\bm{x}_t) + \langle \nabla f(\bm{x}_t), \bm{x}_{t+1}-\bm{x}_t \rangle + \frac{L}{\nu + 1} \| \bm{x}_{t+1} - \bm{x}_t \|^{\nu+1} \\
&= f(\bm{x}_t) - \eta_t \langle \nabla f(\bm{x}_t), \nabla f_{\mathcal{S}_t}(\bm{x}_t)\rangle + \frac{L \eta_t^{\nu+1}}{\nu+1} \| \nabla f_{\mathcal{S}_t}(\bm{x}_t) \|^{\nu+1}.
\end{align*}
By taking the expectation, we have
\begin{align*}
\mathbb{E}\left[ f(\bm{x}_{t+1}) \right]
&\leq \mathbb{E}\left[ f(\bm{x}_t) \right] -\eta_t \mathbb{E}\left[ \| \nabla f(\bm{x}_t) \|^2\right] + \frac{L \eta_t^{\nu+1}}{\nu+1} \mathbb{E}\left[ \| \nabla f_{\mathcal{S}_t}(\bm{x}_t) \|^{\nu+1} \right].
\end{align*}
Furthermore, from Lemma \ref{lem:aux} and $\eta_t^\nu \leq \frac{1}{L}$, 
\begin{align*}
\mathbb{E}\left[ f(\bm{x}_{t+1}) \right]
&\leq \mathbb{E}\left[ f(\bm{x}_t) \right] -\eta_t \mathbb{E}\left[ \| \nabla f(\bm{x}_t) \|^2\right] + \frac{L \eta_t^{\nu+1}}{\nu+1} \left\{ 2^{1-\nu}\left( \frac{C_\mathfrak{p} \sigma^\mathfrak{p}}{b^{\mathfrak{p}-1}} \right)^{\frac{\nu+1}{\mathfrak{p}}} + \frac{\nu+1}{2} \mathbb{E}\left[ \| \nabla f(\bm{x}_t) \|^2 \right] + \frac{1-\nu}{2} \right\} \\
&= \mathbb{E}\left[ f(\bm{x}_t) \right] - \eta_t\left( 1- \frac{L\eta_t^{\nu}}{2} \right) \mathbb{E}\left[ \| \nabla f(\bm{x}_t) \|^2\right] + \frac{L \eta_t^{\nu+1}}{\nu+1} \left\{ 2^{1-\nu} \left( \frac{C_\mathfrak{p} \sigma^\mathfrak{p}}{b^{\mathfrak{p}-1}} \right)^{\frac{\nu+1}{\mathfrak{p}}} + \frac{1-\nu}{2} \right\} \\
&\leq \mathbb{E}\left[ f(\bm{x}_t) \right] - \frac{\eta_t}{2} \mathbb{E}\left[ \| \nabla f(\bm{x}_t) \|^2\right] + \frac{L}{\nu+1} \underbrace{\left\{ 2^{1-\nu} \left( \frac{C_\mathfrak{p} \sigma^\mathfrak{p}}{b^{\mathfrak{p}-1}} \right)^{\frac{\nu+1}{\mathfrak{p}}} + \frac{1-\nu}{2} \right\}}_{=:E_0}\eta_t^{\nu+1}.
\end{align*}
Here, since $f$ is $\mu$-strongly convex, for all $\bm{x} \in \mathbb{R}^d$, we have
\begin{align*}
f(\bm{x}) - f^\star 
&\leq \langle \nabla f(\bm{x}), \bm{x} - \bm{x}^\star\rangle - \frac{\mu}{2}\| \bm{x}-\bm{x}^\star\|^2 \\
&\leq \frac{1}{2\mu}\| \nabla f(\bm{x}) \|^2 + \frac{\mu}{2}\| \bm{x} - \bm{x}^\star \|^2 - \frac{\mu}{2}\| \bm{x}-\bm{x}^\star\|^2 \\
&= \frac{1}{2\mu}\| \nabla f(\bm{x}) \|^2,
\end{align*}
where we have used Young's inequality in the second inequality. Therefore, we have
\begin{align*}
\mathbb{E}\left[ f(\bm{x}_{t+1}) \right]
\leq \mathbb{E}\left[ f(\bm{x}_t) \right] - \mu\eta_t\mathbb{E}\left[ f(\bm{x_t}) - f^\star\right] + \frac{E_0L}{\nu+1} \eta_t^{\nu+1}.
\end{align*}
Subtracting $f^\star$ from both sides yields
\begin{align*}
\mathbb{E}\left[ f(\bm{x}_{t+1}) \right] - f^\star
\leq (1-\mu\eta_t)\mathbb{E}\left[ f(\bm{x}_t) -f^\star\right] + \frac{E_0L}{\nu+1} \eta_t^{\nu+1}.
\end{align*}
Using this inequality, we prove by induction that 
\begin{align*}
\mathbb{E}\left[ f(\bm{x}_t) - f^\star\right] \leq \frac{E_1}{t^\nu}, \text{ where }\  E_1 := \max\left\{ f(\bm{x}_1) -f^\star, \frac{E_0L}{\nu(\nu+1)}\left(\frac{2\nu}{\mu} \right)^{\nu+1}\right\}.
\end{align*}
When $t=1$, from the definition of $E_1$, we have $\mathbb{E}\left[ f(\bm{x}_1) - f^\star \right] \leq E_1$. Assume that $\mathbb{E}\left[f(\bm{x}_t) - f^\star \right] \leq \frac{E_1}{t^\nu}$ holds for some $t \in \mathbb{N}$. Then, from $\eta_t := \frac{2\nu}{\mu t}$, we have
\begin{align*}
\mathbb{E}\left[ f(\bm{x}_{t+1}) - f^\star \right]
&\leq (1-\mu\eta_t)\frac{E_1}{t^\nu} + \frac{E_0L}{\nu+1}\eta_t^{\nu+1} \\
&=\left( 1-\frac{2\nu}{t}\right)\frac{E_1}{t^\nu} + \frac{E_0L}{\nu+1} \left( \frac{2\nu}{\mu}\right)^{\nu+1}\frac{1}{t^{\nu+1}}.
\end{align*}
From the definition of $E_1$, we have $\frac{E_0L}{\nu+1} \left( \frac{2\nu}{\mu}\right)^{\nu+1} \leq \nu E_1$. Then,
\begin{align*}
\mathbb{E}\left[ f(\bm{x}_{t+1}) - f^\star \right]
&\leq \frac{E_1}{t^\nu} - \frac{2\nu E_1}{t^{\nu+1}} + \frac{\nu E_1}{t^{\nu+1}} \\
&= E_1\left( \frac{1}{t^\nu} - \frac{\nu}{t^{\nu+1}} \right).
\end{align*}
Here, since the function $g(t) = t^{-\nu}$ $(\nu \in (0,1])$ is convex, we have $g(t+1) \geq g(t) + g'(t)$, i.e., $\frac{1}{(t+1)^\nu} \geq \frac{1}{t^\nu} - \frac{\nu}{t^{\nu+1}}$. Hence, 
\begin{align*}
\mathbb{E}\left[ f(\bm{x}_{t+1}) - f^\star \right]
\leq \frac{E_1}{(t+1)^\nu}.
\end{align*}
Therefore, for all $t \in \mathbb{N}$, $\mathbb{E}\left[f(\bm{x}_t) -f^\star \right] \leq \frac{E_1}{t^\nu}$ holds. This completes the proof.
\end{proof}

\subsection{Proof of Theorem \ref{thm:sgdm_s}}
\label{sec:sgdm_s}
\begin{lem}\label{lem:sgdm_s_03}
Suppose that Assumptions \ref{assum:01} and \ref{assum:02} hold and $f$ is $\mu$-strongly convex. Let $\Phi_{t} := f(\bm{z}_t) - f^\star + A\|\bm{m}_{t-1}\|^{\nu+1}$. Then, the following inequality holds:
\begin{align*}
\|\nabla f(\bm{x}_t)\|^{2} 
\geq E_2 \Phi_t - E_2 A\| \bm{m}_{t-1}\|^{\nu+1} - \frac{(1+2L)E_2}{2(\nu+1)}\left( \frac{\eta\beta}{1-\beta}\right)^{\nu+1} \| \bm{m}_{t-1}\|^{\nu+1},
\end{align*}
where $E_2 := \frac{2\mu(\nu+1)}{\nu\left(1+L^{\frac{1}{\nu}}\right)+1}$ is a non-negative constant.
\end{lem}

\begin{proof} 
From Lemma \ref{lem:holder}, we have
\begin{align*}
f(\bm{z}_{t}) \le f(\bm{x}_{t}) + \langle \nabla f(\bm{x}_t), \bm{z}_t - \bm{x}_t \rangle + \frac{L}{\nu+1} \|\bm{z}_t - \bm{x}_t\|^{\nu+1}.
\end{align*}
Rearranging the terms and subtracting $f^\star$ from both sides gives
\begin{align*}
f(\bm{x}_{t}) - f^\star \geq f(\bm{z}_t) - f^\star - \langle \nabla f(\bm{x}_t), \bm{z}_t - \bm{x}_t \rangle - \frac{L}{\nu+1} \|\bm{z}_{t} - \bm{x}_{t}\|^{\nu+1}. \nonumber
\end{align*}
Furthermore, multiplying both sides by $2\mu$ yields
\begin{align*}
2\mu (f(\bm{x}_t) - f^\star) \ge 2\mu (f(\bm{z}_t) - f^\star) - 2\mu \langle \nabla f(\bm{x}_t), \bm{z}_t - \bm{x}_t \rangle - \frac{2\mu L}{\nu+1} \|\bm{z}_t - \bm{x}_t\|^{\nu+1}. \nonumber
\end{align*}
Here, from the $\mu$-strong convexity of $f$ and Assumption \ref{assum:01}, we have, for all $\bm{x}, \bm{y} \in \mathbb{R}^d$, 
\begin{align*}
\mu\| \bm{x} - \bm{y} \|^2
\leq\langle \nabla f(\bm{x}) - \nabla f(\bm{y}), \bm{x} - \bm{y} \rangle
\leq \| \nabla f(\bm{x}) - \nabla f(\bm{y}) \| \| \bm{x} - \bm{y} \|
\leq L\| \bm{x} - \bm{y} \|^{\nu+1}.
\end{align*}
That is, we have $\| \bm{x} - \bm{y} \|^{1-\nu} \leq \frac{L}{\mu}$ for all $\nu \in (0,1)$. Therefore, from Assumption \ref{assum:01},
\begin{align*}
\| \nabla f(\bm{x}_t) \|^{\frac{\nu+1}{\nu}}
= \| \nabla f(\bm{x}_t) \|^{2} \| \nabla f(\bm{x}_t) \|^{\frac{1-\nu}{\nu}}
\leq \| \nabla f(\bm{x}_t) \|^{2}L^{\frac{1-\nu}{\nu}}\| \bm{x}_t - \bm{x}^\star \|^{1-\nu}
\leq \frac{L^{\frac{1}{\nu}}}{\mu}\| \nabla f(\bm{x}_t) \|^{2}.
\end{align*}
Note that we arrive at the same result when $\nu = 1$, from $\mu \leq L$. Hence, from the Cauchy-Schwarz and Young’s inequality,
\begin{align*}
-2\mu \langle \nabla f(\bm{x}_t), \bm{z}_t - \bm{x}_t \rangle 
&\geq -2\mu \| \nabla f(\bm{x}_t) \|\| \bm{z}_t - \bm{x}_t \|
\geq -\mu \left( \frac{\nu}{\nu+1}\| \nabla f(\bm{x}_t) \|^{\frac{\nu+1}{\nu}} + \frac{1}{\nu+1}\| \bm{z}_t - \bm{x}_t \|^{\nu+1}\right) \\
&\geq - \frac{\nu L^{\frac{1}{\nu}}}{\nu+1}\| \nabla f(\bm{x}_t) \|^2 - \frac{\mu}{\nu+1} \| \bm{z}_t - \bm{x}_t \|^{\nu+1}.
\end{align*}
Substituting this into the above inequality, we find that
\begin{align*}
2\mu (f(\bm{x}_t) - f^\star) \ge 2\mu (f(\bm{z}_t) - f^\star) - \frac{\nu L^{\frac{1}{\nu}}}{\nu+1}\| \nabla f(\bm{x}_t) \|^2 - \frac{\mu(1+2L)}{\nu+1} \| \bm{z}_t - \bm{x}_t \|^{\nu+1}.
\end{align*}
Since $f$ is $\mu$-strongly convex, the inequality $\|\nabla f(\bm{x}_t)\|^{2} \ge 2\mu (f(\bm{x}_t) - f^\star)$ holds. Then, we have
\begin{align*}
\|\nabla f(\bm{x}_t)\|^{2} \ge 2\mu (f(\bm{z}_t) - f^\star) -  \frac{\nu L^{\frac{1}{\nu}}}{\nu+1}\| \nabla f(\bm{x}_t) \|^2 - \frac{\mu(1+2L)}{\nu+1} \| \bm{z}_t - \bm{x}_t \|^{\nu+1}.
\end{align*}
Isolating $\|\nabla f(\bm{x}_t)\|^{2}$ and organizing the terms lead to 
\begin{align*}
\|\nabla f(\bm{x}_t)\|^{2} \ge \frac{2\mu(\nu+1)}{\nu\left(1+L^{\frac{1}{\nu}}\right)+1} (f(\bm{z}_t) - f^\star) - \frac{\mu(1+2L)}{\nu\left(1+L^{\frac{1}{\nu}} \right) + 1} \| \bm{z}_t - \bm{x}_t \|^{\nu+1}.
\end{align*}
From the definition of $\Phi_{t}$ and $\bm{z}_t - \bm{x}_t = - \frac{\eta \beta}{1-\beta} \bm{m}_{t-1}$,
\begin{align*}
\|\nabla f(\bm{x}_t)\|^{2} 
&\ge \underbrace{\frac{2\mu(\nu+1)}{\nu\left(1+L^{\frac{1}{\nu}}\right)+1}}_{=:E_2} \Phi_{t}- \frac{2\mu(\nu+1)A}{\nu\left(1+L^{\frac{1}{\nu}}\right)+1} \|\bm{m}_{t-1}\|^{\nu+1}
- \frac{\mu(1+2L)}{\nu\left(1+L^{\frac{1}{\nu}} \right) + 1}\left( \frac{\eta\beta}{1-\beta}\right)^{\nu+1} \| \bm{m}_{t-1}\|^{\nu+1} \\
&\geq E_2 \Phi_t - E_2 A\| \bm{m}_{t-1}\|^{\nu+1} - \frac{(1+2L)E_2}{2(\nu+1)}\left( \frac{\eta\beta}{1-\beta}\right)^{\nu+1} \| \bm{m}_{t-1}\|^{\nu+1}.
\end{align*}
This completes the proof.
\end{proof}

\begin{lem}\label{lem:sgdm_s_01}
Suppose that Assumption \ref{assum:02} holds. Then,
\begin{align*}
&\mathbb{E}\left[ \| \bm{m}_{t} \|^{\nu+1} \right] - \mathbb{E}\left[ \| \bm{m}_{t-1} \|^{\nu+1} \right]\\
&\quad\leq -(1-\beta)\mathbb{E}\left[ \| \bm{m}_{t-1} \|^{\nu+1} \right] + \frac{(1-\beta)(\nu+1)}{2} \mathbb{E}\left[ \| \nabla f(\bm{x}_t) \|^{2} \right] + \frac{(1-\beta)(1-\nu)}{2} + 2^{1-\nu}(1-\beta)^{\nu+1}\left( \frac{C_\mathfrak{p} \sigma^\mathfrak{p}}{b^{\mathfrak{p}-1}}\right)^{\frac{\nu+1}{\mathfrak{p}}}.
\end{align*}
\end{lem}

\begin{proof}
From the definition of $\bm{m}_t$ and Lemma \ref{lem:q-norm}, we have
\begin{align*}
\| \bm{m}_t \|^{\nu+1}
&= \| \beta \bm{m}_{t-1} + (1-\beta) \nabla f_{\mathcal{S}_t}(\bm{x}_t) \|^{\nu+1} \\
&= \| \left(\beta \bm{m}_{t-1} + (1-\beta) \nabla f(\bm{x}_t)\right) + (1-\beta)(\nabla f_{\mathcal{S}_t}(\bm{x}_t) - \nabla f(\bm{x}_t))\|^{\nu+1} \\
&= \| \beta \bm{m}_{t-1} + (1-\beta) \nabla f(\bm{x}_t) \|^{\nu+1} + 2^{1-\nu}(1-\beta)^{\nu+1} \| \nabla f_{\mathcal{S}_t}(\bm{x}_t) - \nabla f(\bm{x}_t) \|^{\nu+1} \\
&\quad + (\nu+1)(1-\beta)\| \beta \bm{m}_{t-1}+(1-\beta)\nabla f(\bm{x}_t)\|^{\nu-1}\langle \beta \bm{m}_{t-1} + (1-\beta) \nabla f(\bm{x}_t), \nabla f_{\mathcal{S}_t}(\bm{x}_t) - \nabla f(\bm{x}_t)\rangle.
\end{align*}
From Assumption \ref{assum:02}, we find, by taking the expectation, that
\begin{align*}
\mathbb{E}\left[ \| \bm{m}_t \|^{\nu+1} \right]
&= \mathbb{E}\left[ \| \beta \bm{m}_{t-1} + (1-\beta) \nabla f(\bm{x}_t) \|^{\nu+1}\right] + 2^{1-\nu}(1-\beta)^{\nu+1}\left( \frac{C_\mathfrak{p} \sigma^\mathfrak{p}}{b^{\mathfrak{p}-1}}\right)^{\frac{\nu+1}{\mathfrak{p}}} \\
&\leq \beta \mathbb{E}\left[ \| \bm{m}_{t-1} \|^{\nu+1}\right] + (1-\beta)\mathbb{E}\left[ \| \nabla f(\bm{x}_t) \|^{\nu+1}\right] + 2^{1-\nu}(1-\beta)^{\nu+1}\left( \frac{C_\mathfrak{p} \sigma^\mathfrak{p}}{b^{\mathfrak{p}-1}}\right)^{\frac{\nu+1}{\mathfrak{p}}}.
\end{align*}
Subtracting $ \mathbb{E}\left[ \| \bm{m}_{t-1} \|^{\nu+1}\right]$ from both sides gives
\begin{align*}
&\mathbb{E}\left[ \| \bm{m}_{t} \|^{\nu+1} \right] - \mathbb{E}\left[ \| \bm{m}_{t-1} \|^{\nu+1} \right]\\
&\quad\leq -(1-\beta)\mathbb{E}\left[ \| \bm{m}_{t-1} \|^{\nu+1} \right] + (1-\beta) \mathbb{E}\left[ \| \nabla f(\bm{x}_t) \|^{\nu+1} \right] + 2^{1-\nu}(1-\beta)^{\nu+1}\left( \frac{C_\mathfrak{p} \sigma^\mathfrak{p}}{b^{\mathfrak{p}-1}}\right)^{\frac{\nu+1}{\mathfrak{p}}}.
\end{align*}
From Young's inequality, we have
\begin{align*}
\| \nabla f(\bm{x}_t) \|^{\nu+1} \cdot 1
\leq \frac{\nu+1}{2}\| \nabla f(\bm{x}_t) \|^2 + \frac{1-\nu}{2}.
\end{align*}
Therefore, we have
\begin{align*}
&\mathbb{E}\left[ \| \bm{m}_{t} \|^{\nu+1} \right] - \mathbb{E}\left[ \| \bm{m}_{t-1} \|^{\nu+1} \right]\\
&\quad\leq -(1-\beta)\mathbb{E}\left[ \| \bm{m}_{t-1} \|^{\nu+1} \right] + \frac{(1-\beta)(\nu+1)}{2} \mathbb{E}\left[ \| \nabla f(\bm{x}_t) \|^{2} \right] + \frac{(1-\beta)(1-\nu)}{2} + 2^{1-\nu}(1-\beta)^{\nu+1}\left( \frac{C_\mathfrak{p} \sigma^\mathfrak{p}}{b^{\mathfrak{p}-1}}\right)^{\frac{\nu+1}{\mathfrak{p}}}.
\end{align*}
This completes the proof.
\end{proof}

\begin{lem}\label{lem:sgdm_s_02}
Let $\bm{z}_t := \bm{x}_t - \frac{\beta \eta}{1-\beta} \bm{m}_{t-1}$. Suppose that Assumptions \ref{assum:01} and \ref{assum:02} hold. Then,
\begin{align*}
\mathbb{E}[f(\bm{z}_{t+1})] 
&\leq \mathbb{E}[f(\bm{z}_t)] -\left( \frac{\eta}{2} - \frac{L\eta^{\nu+1}}{2}\right)\mathbb{E}[\| \nabla f(\bm{x}_t) \|^2]
+ \frac{\eta L^2 \nu}{\nu+1}\left( \frac{\eta\beta}{1-\beta}\right)^{2\nu}\mathbb{E}\left[\| \bm{m}_{t-1} \|^{\nu+1} \right] \\
&\quad+ \frac{\eta(1-\nu)L^2}{2(\nu+1)}\left( \frac{\eta\beta}{1-\beta}\right)^{2\nu}+\frac{L\eta^{\nu+1}}{\nu+1} \left\{ 2^{1-\nu} \left( \frac{C_\mathfrak{p} \sigma^\mathfrak{p}}{b^{\mathfrak{p}-1}} \right)^{\frac{\nu+1}{\mathfrak{p}}} + \frac{1-\nu}{2} \right\}.
\end{align*}
\end{lem}

\begin{proof}
From the definition of $\bm{z}_t$, we have
\begin{align*}
\bm{z}_{t+1}
&= \bm{x}_{t+1} - \frac{\eta\beta}{1-\beta}\bm{m}_t
= (\bm{x}_{t} - \eta \bm{m}_t) - \frac{\eta\beta}{1-\beta}\bm{m}_t
= \bm{x}_{t} - \frac{\eta}{1-\beta}\bm{m}_t \\
&= \bm{x}_t - \frac{\eta}{1-\beta}\left\{ \beta \bm{m}_{t-1} + (1-\beta)\nabla f_{\mathcal{S}_t}(\bm{x}_t)\right\}
= \bm{z}_t - \eta\nabla f_{\mathcal{S}_t}(\bm{x}_t).
\end{align*}
From Assumption \ref{assum:01},
\begin{align*}
f(\bm{z}_{t+1})
&\leq f(\bm{z}_t) + \langle \nabla f(\bm{z}_t), \bm{z}_{t+1} - \bm{z}_t \rangle + \frac{L}{\nu + 1} \| \bm{z}_{t+1} - \bm{z}_t \|^{\nu+1} \\
&= f(\bm{z}_t) - \eta \langle \nabla f(\bm{z}_t), \nabla f_{\mathcal{S}_t}(\bm{x}_t) \rangle + \frac{L\eta^{\nu+1}}{\nu+1} \| \nabla f_{\mathcal{S}_t}(\bm{x}_t) \|^{\nu+1}.
\end{align*}
By taking the expectation, we have
\begin{align*}
\mathbb{E}[f(\bm{z}_{t+1})] \leq \mathbb{E}[f(\bm{z}_t)] - \eta \mathbb{E}[ \langle \nabla f(\bm{z}_t), \nabla f(\bm{x}_t) \rangle ] + \frac{L\eta^{\nu+1}}{\nu+1}\mathbb{E}[ \| \nabla f_{\mathcal{S}_t}(\bm{x}_t) \|^{\nu+1} ].
\end{align*}
From $\| \bm{x} - \bm{y} \|^2 = \| \bm{x} \|^2 -2\langle \bm{x}, \bm{y}\rangle + \| \bm{y} \|^2$ and Lemma \ref{lem:aux},
\begin{align*}
\mathbb{E}[f(\bm{z}_{t+1})] 
&\leq \mathbb{E}[f(\bm{z}_t)] - \frac{\eta}{2} ( \mathbb{E}[\| \nabla f(\bm{z}_t) \|^2] + \mathbb{E}[\| \nabla f(\bm{x}_t) \|^2] - \mathbb{E}[\| \nabla f(\bm{z}_t) - \nabla f(\bm{x}_t) \|^2] ) \\
&\quad+ \frac{L\eta^{\nu+1}}{\nu+1} \left( 2^{1-\nu} \left( \frac{C_\mathfrak{p} \sigma^\mathfrak{p}}{b^{\mathfrak{p}-1}} \right)^{\frac{\nu+1}{\mathfrak{p}}} + \frac{\nu+1}{2}\mathbb{E}[\| \nabla f(\bm{x}_t) \|^2] + \frac{1-\nu}{2} \right) \\
&\leq \mathbb{E}[f(\bm{z}_t)] - \left( \frac{\eta}{2} - \frac{L\eta^{\nu+1}}{2} \right)\mathbb{E}[\| \nabla f(\bm{x}_t) \|^2] + \frac{\eta}{2}\mathbb{E}[\| \nabla f(\bm{z}_t) - \nabla f(\bm{x}_t) \|^2] \\
&\quad + \frac{L\eta^{\nu+1}}{\nu+1} \left\{ 2^{1-\nu} \left( \frac{C_\mathfrak{p} \sigma^\mathfrak{p}}{b^{\mathfrak{p}-1}} \right)^{\frac{\nu+1}{\mathfrak{p}}} + \frac{1-\nu}{2} \right\}.
\end{align*}
Here, from Assumption \ref{assum:01}, the definition of $\bm{z}_t$, and Young's inequality, we have
\begin{align*}
\frac{\eta}{2} \mathbb{E}[\| \nabla f(\bm{z}_t) - \nabla f(\bm{x}_t) \|^2] 
&\leq \frac{\eta}{2}L^2\mathbb{E}[\| \bm{z}_t - \bm{x}_t \|^{2\nu}] \\
&=\frac{\eta}{2}L^2\left(\frac{\eta\beta}{1-\beta} \right)^{2\nu}\mathbb{E}\left[ \| \bm{m}_{t-1}\|^{2\nu}\right] \\
&\leq \frac{\eta}{2}L^2\left(\frac{\eta\beta}{1-\beta} \right)^{2\nu} \left\{ \frac{2\nu}{\nu+1}\mathbb{E}\left[ \| \bm{m}_{t-1} \|^{\nu+1}\right] + \frac{1-\nu}{\nu+1} \right\} \\
&= \frac{\eta L^2 \nu}{\nu+1}\left( \frac{\eta\beta}{1-\beta}\right)^{2\nu}\mathbb{E}\left[\| \bm{m}_{t-1} \|^{\nu+1} \right] + \frac{\eta(1-\nu)L^2}{2(\nu+1)}\left( \frac{\eta\beta}{1-\beta}\right)^{2\nu}.
\end{align*}
Therefore,
\begin{align*}
\mathbb{E}[f(\bm{z}_{t+1})] 
&\leq \mathbb{E}[f(\bm{z}_t)] -\left( \frac{\eta}{2} - \frac{L\eta^{\nu+1}}{2}\right)\mathbb{E}[\| \nabla f(\bm{x}_t) \|^2]
+ \frac{\eta L^2 \nu}{\nu+1}\left( \frac{\eta\beta}{1-\beta}\right)^{2\nu}\mathbb{E}\left[\| \bm{m}_{t-1} \|^{\nu+1} \right] \\
&\quad+ \frac{\eta(1-\nu)L^2}{2(\nu+1)}\left( \frac{\eta\beta}{1-\beta}\right)^{2\nu}+\frac{L\eta^{\nu+1}}{\nu+1} \left\{ 2^{1-\nu} \left( \frac{C_\mathfrak{p} \sigma^\mathfrak{p}}{b^{\mathfrak{p}-1}} \right)^{\frac{\nu+1}{\mathfrak{p}}} + \frac{1-\nu}{2} \right\}.
\end{align*}
This completes the proof.
\end{proof}

\begin{manualthm_sgdm_s} 
Let the function $f \colon \mathbb{R}^d \to \mathbb{R}$ be $\mu$-strongly convex. Suppose that Assumptions \ref{assum:01} and \ref{assum:02} hold. With learning rate $\eta_t := \eta$, if $\nu+1 \leq \mathfrak{p}$ and $\eta \leq \min\left\{ \frac{2(1-\beta)}{E_2}, \left( \frac{1}{8L} \right)^{\frac{1}{\nu}}, \left( \frac{(1-\beta)^{\nu+1}}{2E_2(1+2L)\beta^{\nu+1}}\right)^{\frac{1}{\nu+1}}, \left(\frac{(1-\beta)^{2\nu}}{2L^2\nu\beta^{2\nu}} \right)^{\frac{1}{2\nu}} \right\}$, then:
\begin{align*}
\mathbb{E}\left[ f(\bm{x}_T) - f^\star \right]
&\leq \frac{f(\bm{x}_1) - f^\star}{2} \left(1-\frac{\eta E_2}{4}\right)^T + \frac{E_3}{1-\beta} 
=\mathcal{O}\left( \left(1-\frac{\eta E_2}{4}\right)^T + \eta^{\nu+1}\right).
\end{align*}
In particular, when $\nu + 1 = \mathfrak{p}$,
\begin{align*}
\mathbb{E}\left[ f(\bm{x}_T) - f^\star \right]
&=\mathcal{O}\left( \left(1-\frac{\eta E_2}{4}\right)^T + \eta^{\mathfrak{p}}\right),
\end{align*}
where $E_2 := \frac{2\mu(\nu+1)}{\nu\left(1+L^{\frac{1}{\nu}}\right)+1}$, $E_3 := \frac{LE_0}{\nu+1} + \frac{(1+2L)E_0E_2}{4(\nu+1)}\left(\frac{\beta}{1-\beta} \right)^{\nu+1}\eta + \frac{2L^2\nu(1-\beta) E_0}{\nu+1}\left( \frac{\beta}{1-\beta}\right)^{2\nu}\eta^{\nu}$, and $E_0$ defined in Eq. \eqref{eq:E0} are non-negative constants.
\end{manualthm_sgdm_s}

\begin{proof} 
For all $t \in \mathbb{N}$, let $\Phi_{t} := f(\bm{z}_t) - f^\star + A\|m_{t-1}\|^{\nu+1}$. Then, 
\begin{align*}
\mathbb{E}[\Phi_{t+1}] - \mathbb{E}[\Phi_{t}] = \mathbb{E}[f(z_{t+1})] - \mathbb{E}[f(\bm{z}_t)] + A(\mathbb{E}[\|m_{t}\|^{\nu+1}] - \mathbb{E}[\|m_{t-1}\|^{\nu+1}]). \nonumber
\end{align*}
From Lemma \ref{lem:sgdm_s_02}, we have
\begin{align*}
\mathbb{E}[f(\bm{z}_{t+1})] - \mathbb{E}[f(\bm{z}_t)]  
&\leq -\left( \frac{\eta}{2} - \frac{L\eta^{\nu+1}}{2}\right)\mathbb{E}[\| \nabla f(\bm{x}_t) \|^2]
+ \frac{\eta L^2 \nu}{\nu+1}\left( \frac{\eta\beta}{1-\beta}\right)^{2\nu}\mathbb{E}\left[\| \bm{m}_{t-1} \|^{\nu+1} \right] \\
&\quad+ \frac{\eta(1-\nu)L^2}{2(\nu+1)}\left( \frac{\eta\beta}{1-\beta}\right)^{2\nu}+\frac{L\eta^{\nu+1}}{\nu+1} \underbrace{\left\{ 2^{1-\nu} \left( \frac{C_\mathfrak{p} \sigma^\mathfrak{p}}{b^{\mathfrak{p}-1}} \right)^{\frac{\nu+1}{\mathfrak{p}}} + \frac{1-\nu}{2} \right\}}_{=:E_0}.
\end{align*}
Furthermore, from Lemma \ref{lem:sgdm_s_01} and Young’s inequality,
\begin{align*}
A(\mathbb{E}[\|m_{t}\|^{\nu+1}] - \mathbb{E}[\|m_{t-1}\|^{\nu+1}]) 
\le& -A(1-\beta)\mathbb{E}[\|m_{t-1}\|^{\nu+1}] + \frac{A(1-\beta)(\nu+1)}{2} \mathbb{E}[\|\nabla f(\bm{x}_t)\|^{2}] \\
&+ \frac{A(1-\beta)(1-\nu)}{2} + 2^{1-\nu}(1-\beta)^{\nu+1} A\left( \frac{C_{\mathfrak{p}}\sigma^{\mathfrak{p}}}{b^{\mathfrak{p}-1}} \right)^{\frac{\nu+1}{\mathfrak{p}}}.
\end{align*}
Combining these results, 
\begin{align*}
\mathbb{E}[\Phi_{t+1}] - \mathbb{E}[\Phi_{t}]
&\leq-\left( \frac{\eta}{2} - \frac{L\eta^{\nu+1}}{2}\right)\mathbb{E}[\| \nabla f(\bm{x}_t) \|^2]
+ \frac{\eta L^2 \nu}{\nu+1}\left( \frac{\eta\beta}{1-\beta}\right)^{2\nu}\mathbb{E}\left[\| \bm{m}_{t-1} \|^{\nu+1} \right] \\
&\quad+ \frac{\eta(1-\nu)L^2}{2(\nu+1)}\left( \frac{\eta\beta}{1-\beta}\right)^{2\nu}+\frac{L\eta^{\nu+1}}{\nu+1} E_0 \\
&\quad-A(1-\beta)\mathbb{E}[\|m_{t-1}\|^{\nu+1}] + \frac{A(1-\beta)(\nu+1)}{2} \mathbb{E}[\|\nabla f(\bm{x}_t)\|^{2}] \\
&\quad+ \frac{A(1-\beta)(1-\nu)}{2} + 2^{1-\nu}(1-\beta)^{\nu+1} A\left( \frac{C_{\mathfrak{p}}\sigma^{\mathfrak{p}}}{b^{\mathfrak{p}-1}} \right)^{\frac{\nu+1}{\mathfrak{p}}}.
\end{align*}
Next, we split the gradient term $-\frac{\eta}{2}\mathbb{E}[\|\nabla f(\bm{x}_t)\|^{2}]$ into two $-\frac{\eta}{4}$ parts and apply the result of Lemma \ref{lem:sgdm_s_03} to the first part,
\begin{align*}
&\mathbb{E}[\Phi_{t+1}] - \mathbb{E}[\Phi_{t}] \\
&\quad\le -\frac{\eta}{4} \left\{ E_2 \mathbb{E}\left[\Phi_t\right] - E_2 A\mathbb{E}\left[\| \bm{m}_{t-1}\|^{\nu+1}\right] - \frac{(1+2L)E_2}{2(\nu+1)}\left( \frac{\eta\beta}{1-\beta}\right)^{\nu+1} \mathbb{E}\left[\| \bm{m}_{t-1}\|^{\nu+1}\right]\right\} \\
&\quad\quad- \left( \frac{\eta}{4} - \frac{L\eta^{\nu+1}}{2} - \frac{A(1-\beta)(1+\nu)}{2} \right) \mathbb{E}[\|\nabla f(\bm{x}_t)\|^{2}] 
+ \left\{ \frac{\eta L^2 \nu}{\nu+1}\left( \frac{\eta\beta}{1-\beta}\right)^{2\nu} - A(1-\beta) \right\} \mathbb{E}[\|m_{t-1}\|^{\nu+1}] \\
&\quad\quad+ \frac{L\eta^{\nu+1}}{\nu+1} E_0 + \frac{\eta(1-\nu)L^2}{2(\nu+1)}\left( \frac{\eta\beta}{1-\beta}\right)^{2\nu}
+ A \left\{ 2^{1-\nu}(1-\beta)^{\nu+1} \left( \frac{C_{\mathfrak{p}}\sigma^{\mathfrak{p}}}{b^{\mathfrak{p}-1}} \right)^{\frac{\nu+1}{\mathfrak{p}}} + \frac{(1-\beta)(1-\nu)}{2} \right\} \\
&\quad\leq -\frac{\eta E_2}{4}\mathbb{E}\left[\Phi_t\right]
+ \left\{ \frac{\eta E_2}{4}A + \frac{\eta(1+2L)E_2}{8(\nu+1)}\left( \frac{\eta\beta}{1-\beta}\right)^{\nu+1} + \frac{\eta L^2 \nu}{\nu+1}\left( \frac{\eta\beta}{1-\beta}\right)^{2\nu} - A(1-\beta)\right\}\mathbb{E}[\|m_{t-1}\|^{\nu+1}] \\
&\quad\quad - \left( \frac{\eta}{4} - \frac{L\eta^{\nu+1}}{2} - \frac{A(1-\beta)(1+\nu)}{2} \right) \mathbb{E}[\|\nabla f(\bm{x}_t)\|^{2}] + \frac{L\eta^{\nu+1}}{\nu+1} E_0 \\
&\quad\quad+ A \left\{ 2^{1-\nu}(1-\beta)^{\nu+1} \left( \frac{C_{\mathfrak{p}}\sigma^{\mathfrak{p}}}{b^{\mathfrak{p}-1}} \right)^{\frac{\nu+1}{\mathfrak{p}}} + \frac{(1-\beta)(1-\nu)}{2} \right\},
\end{align*}
where $E_2 := \frac{2\mu(\nu+1)}{\nu\left(1+L^{\frac{1}{\nu}}\right)+1}$. To eliminate the terms involving $\mathbb{E}[\|m_{t-1}\|^{\nu+1}]$, we define the coefficient $A$ as follows:
\begin{align*}
A = \frac{\frac{\eta (1+2L) E_2}{8(\nu+1)} \left( \frac{\eta \beta}{1-\beta} \right)^{\nu+1} + \frac{\eta L^2 \nu}{\nu+1} \left( \frac{\eta \beta}{1-\beta} \right)^{2\nu}}{(1-\beta) - \frac{\eta E_2}{4}}.
\end{align*}
Since $\eta \leq \frac{2(1-\beta)}{E_2}$, $A > 0$ always holds and we have
\begin{align*}
A \leq \frac{\eta(1+2L)E_2}{4(1-\beta)(\nu+1)}\left( \frac{\eta\beta}{1-\beta}\right)^{\nu+1} + \frac{2\eta L^2\nu}{(\nu+1)(1-\beta)}\left( \frac{\eta\beta}{1-\beta}\right)^{2\nu}.
\end{align*}
Let us show that the coefficient of $\mathbb{E}\left[ \| \nabla f(\bm{x}_t) \|^2\right]$ is negative.
\begin{align*}
\frac{\eta}{4} - \frac{L\eta^{\nu+1}}{2} - \frac{A(1-\beta)(1+\nu)}{2} 
&\geq \frac{\eta}{4} - \frac{L\eta^{\nu+1}}{2} - \frac{\eta}{8}(1+2L)E_2\left(\frac{\eta\beta}{1-\beta} \right)^{\nu+1} - (1-\beta)\eta L^2 \nu\left( \frac{\eta\beta}{1-\beta}\right)^{2\nu} \\
&=\frac{\eta}{8}\left\{ 2 - 4L\eta^{\nu} - (1+2L)E_2 \left(\frac{\beta}{1-\beta} \right)^{\nu+1}\eta^{\nu+1} - (1-\beta)L^2 \nu \left( \frac{\beta}{1-\beta}\right)^{2\nu}\eta^{2\nu}\right\}.
\end{align*}
Since $\eta \le \min \left\{ \left(\frac{1}{8L}\right)^{\frac{1}{\nu}},\left[ \frac{(1-\beta)^{\nu+1}}{2 E_2 (1 + 2L) \beta^{\nu+1}} \right]^{\frac{1}{\nu+1}},\left[ \frac{(1-\beta)^{2\nu}}{2L^{2}\nu\beta^{2\nu}} \right]^{\frac{1}{2\nu}}\right\}$, we have
\begin{align*}
\eta^\nu \leq \frac{1}{8L}, 
\ \eta^{\nu+1} \leq \frac{1}{2E_2(1+2L)}\left(\frac{1-\beta}{\beta} \right)^{\nu+1}, 
\ \text{and } \ \eta^{2\nu} 
\leq \frac{1}{2L^2\nu(1-\beta)}\left(\frac{1-\beta}{\beta}\right)^{2\nu}.
\end{align*}
Hence, 
\begin{align*}
\frac{\eta}{4} - \frac{L\eta^{\nu+1}}{2} - \frac{A(1-\beta)(1+\nu)}{2}
\geq \frac{\eta}{8}\left( 2 - \frac{1}{2} - \frac{1}{2} - \frac{1}{2}\right)
=\frac{\eta}{16}.
\end{align*}
Therefore, we have
\begin{align*}
\mathbb{E}[\Phi_{t+1}]
&\le \left( 1 - \frac{\eta{E_2}}{4} \right) \mathbb{E}[\Phi_{t}] + \frac{L\eta^{\nu+1}}{\nu+1}E_0 + A \left\{ 2^{1-\nu}(1-\beta)^{\nu+1} \left( \frac{C_{\mathfrak{p}}\sigma^{\mathfrak{p}}}{b^{\mathfrak{p}-1}} \right)^{\frac{\nu+1}{\mathfrak{p}}} + \frac{(1-\beta)(1-\nu)}{2} \right\} \\
&\leq \left( 1 - \frac{\eta{E_2}}{4} \right) \mathbb{E}[\Phi_{t}] + \frac{L\eta^{\nu+1}}{\nu+1}E_0 + A (1-\beta)E_0 \\
&\leq \underbrace{\left( 1 - \frac{\eta{E_2}}{4} \right)}_{=:\rho} \mathbb{E}[\Phi_{t}] + \underbrace{\left\{\frac{LE_0}{\nu+1} + \frac{(1+2L)E_0E_2}{4(\nu+1)}\left(\frac{\beta}{1-\beta} \right)^{\nu+1}\eta + \frac{2L^2\nu(1-\beta) E_0}{\nu+1}\left( \frac{\beta}{1-\beta}\right)^{2\nu}\eta^{\nu}\right\}}_{=:E_3}\eta^{\nu+1}.
\end{align*}
Since $\eta \leq \frac{2(1-\beta)}{E_2}$, we have $|\rho| < 1$. Then,
\begin{align*}
\mathbb{E}\left[ \Phi_{t+1} \right]
\leq \rho^t \mathbb{E}\left[ \Phi_1\right] + \frac{E_3}{1-\rho}
= \rho^{t+1} \cdot \frac{f(\bm{x}_1) - f^\star}{\rho} + \frac{E_3 \eta^{\nu+1}}{1-\rho}.
\end{align*}
Hence,
\begin{align*}
\mathbb{E}\left[ f(\bm{z}_t) - f^\star \right]
\leq \rho^t \cdot \frac{f(\bm{x}_1) - f^\star}{\rho} + \frac{E_3\eta^{\nu+1}}{1-\rho}.
\end{align*}
Here, we can expand $\bm{m}_{t-1} = (1-\beta) \sum_{k=1}^{t-1} \beta^{t-1-k} \nabla f_{\mathcal{S}_k}(\bm{x}_k)$. Then, by the definition of $\bm{z}_t$, 
\begin{align*}
\bm{x}_t &= \bm{z}_t + \frac{\eta\beta}{1-\beta} \bm{m}_{t-1}
= \bm{z}_t + \eta \sum_{k=1}^{t-1} \beta^{t-k} \nabla f_{\mathcal{S}_k}(\bm{x}_k).
\end{align*}
Substituting $\eta \nabla f_{\mathcal{S}_k}(\bm{x}_k) = \bm{z}_k - \bm{z}_{k+1}$ gives
\begin{align*}
\bm{x}_t &= \bm{z}_t + \sum_{k=1}^{t-1} \beta^{t-k} (\bm{z}_k - \bm{z}_{k+1}) \\
&= \bm{z}_t + \left( \sum_{k=1}^{t-1} \beta^{t-k} \bm{z}_k - \sum_{k=1}^{t-1} \beta^{t-k} \bm{z}_{k+1} \right) \\
&= \bm{z}_t + \beta^{t-1} \bm{z}_1 + \sum_{k=2}^{t-1} \beta^{t-k} \bm{z}_k - \left( \sum_{k=2}^{t-1} \beta^{t-k+1} \bm{z}_k + \beta \bm{z}_t \right) \\
&= (1-\beta) \bm{z}_t + \sum_{k=2}^{t-1} \beta^{t-k}(1-\beta) \bm{z}_k + \beta^{t-1} \bm{z}_1.
\end{align*}
Finally, we have
\begin{align*}
(1-\beta) + (1-\beta) \sum_{k=2}^{t-1} \beta^{t-k} + \beta^{t-1}
&= (1-\beta) + (1-\beta) \frac{\beta(1-\beta^{t-2})}{1-\beta} + \beta^{t-1} \\
&= 1 - \beta + \beta - \beta^{t-1} + \beta^{t-1} \\
&= 1.
\end{align*}
Since $\beta \in [0, 1)$, $\bm{x}_t$ can be expressed as a convex combination with the auxiliary sequence $\{\bm{z}_k\}_{k=1}^t$. From the convexity of $f$, we have
\begin{align*}
\mathbb{E}\left[ f(\bm{x}_t) - f^\star \right] = \mathbb{E}\left[ f\left( \sum_{k=1}^t w_k \bm{z}_k \right) - f^\star \right]
\leq \sum_{k=1}^t w_k \mathbb{E}\left[ f(\bm{z}_k) - f^\star \right]
\leq \frac{f(\bm{x}_1)-f^\star}{\rho}\sum_{k=1}^{t}w_k \rho^k + \frac{E_3\eta^{\nu+1}}{1-\rho},
\end{align*}
where $w_1 := \beta^{t-1}, w_t := 1-\beta$, and $w_k := (1-\beta)\beta^{t-k}$ $(1 < k < t$). Since $\eta \leq \frac{2(1-\beta)}{E_2}$, we have $\beta < \rho$. Then,
\begin{align*}
\sum_{k=1}^t w_k \rho^k
=\beta^{t-1}\rho + (1-\beta)\sum_{k=2}^t \beta^{t-k}\rho^k
=\beta^{t-1}\rho + (1-\beta)\rho^t\sum_{j=0}^{t-2}\left(\frac{\beta}{\rho} \right)^j
\leq \frac{1-\beta}{\rho-\beta}\rho^{t+1}.
\end{align*}
Finally, we have
\begin{align*}
\mathbb{E}\left[ f(\bm{x}_t) - f^\star \right]
&\leq \frac{(1-\beta)(f(\bm{x}_1) - f^\star)}{\rho-\beta}\rho^t + \frac{E_3\eta^{\nu+1}}{1-\beta} \\
&\leq \frac{f(\bm{x}_1) - f^\star}{2} \left(1-\frac{\eta E_2}{4}\right)^t + \frac{E_3\eta^{\nu+1}}{1-\beta} \\
&=\mathcal{O}\left( \left(1-\frac{\eta E_2}{4}\right)^t + \eta^{\nu+1}\right),
\end{align*}
where $E_2 := \frac{2\mu(\nu+1)}{\nu\left(1+L^{\frac{1}{\nu}}\right)+1}$ and $E_3 := \frac{LE_0}{\nu+1} + \frac{(1+2L)E_0E_2}{4(\nu+1)}\left(\frac{\beta}{1-\beta} \right)^{\nu+1}\eta + \frac{2L^2\nu(1-\beta) E_0}{\nu+1}\left( \frac{\beta}{1-\beta}\right)^{2\nu}\eta^{\nu}$. This completes the proof.
\end{proof}

\subsection{Proof of Theorem \ref{thm:sgd_c}}
\label{sec:sgd_c}
\begin{lem}\label{lem:lr}
Let $\eta_t := \eta_{\max}t^{-\frac{1}{\nu+1}}$ be the learning rate schedule. Then, 
\begin{align*}
\frac{1}{\sum_{t=1}^{T}\eta_t}
\leq \frac{\frac{\nu}{\nu+1}}{\eta_{\max} \left\{ (T+1)^{\frac{\nu}{\nu+1}} -1 \right\}}
=\mathcal{O}\left( \frac{1}{T^{\frac{\nu}{\nu+1}}}\right),
\end{align*}
and 
\begin{align*}
\frac{\sum_{t=1}^{T}\eta_t^{\nu+1}}{\sum_{t=1}^{T}\eta_t}
\leq \frac{\frac{\nu}{\nu+1}\left( 1 + \log T \right)}{(T+1)^{\frac{\nu}{\nu+1}} -1}
= \mathcal{O}\left( \frac{\log T}{T^{\frac{\nu}{\nu+1}}}\right).
\end{align*}
\end{lem}

\begin{proof}
Here, we have
\begin{align*}
\sum_{t=1}^{T} \eta_t 
= \eta_{\max} \sum_{t=1}^{T} t^{-\frac{1}{\nu+1}}
\geq \eta_{\max} \int_{1}^{T+1} t^{-\frac{1}{\nu+1}} 
= \eta_{\max}\frac{(T+1)^{1-\frac{1}{\nu+1}}-1}{1-\frac{1}{\nu+1}}
= \eta_{\max}\frac{(T+1)^{\frac{\nu}{\nu+1}}-1}{\frac{\nu}{\nu+1}}.
\end{align*}
Therefore, 
\begin{align*}
\frac{1}{\sum_{t=1}^{T} \eta_t} 
\leq \frac{\frac{\nu}{\nu+1}}{\eta_{\max} \left\{ (T+1)^{\frac{\nu}{\nu+1}} -1 \right\}}.
\end{align*}
In addition, 
\begin{align*}
\sum_{t=1}^{T} \eta_t^{\nu+1} 
= \eta_{\max} \sum_{t=1}^{T} t^{-1}
\leq \eta_{\max}\left( 1+\log T \right).
\end{align*}
This completes the proof.
\end{proof}

\begin{manualthm_sgd_c}
Let the function $f \colon \mathbb{R}^d \to \mathbb{R}$ be convex. Suppose that Assumptions \ref{assum:01}, \ref{assum:02}, and \ref{assum:03}(i) hold. With learning rate $\eta_t := \eta_{\max}t^{-\frac{1}{\nu+1}}$, if $\nu+1\leq \mathfrak{p}$, then:
\begin{align*}
\min_{1\leq t\leq T} \mathbb{E}\left[ f(\bm{x}_t) - f^\star\right] 
\leq \frac{\| \bm{x}_1 - \bm{x}^\star\|^{\nu+1}}{(\nu+1)D_1^{\nu-1}}\frac{1}{\sum_{t=1}^{T}\eta_t} +E_4 \frac{\sum_{t=1}^{T}\eta_t^{\nu+1}}{\sum_{t=1}^{T}\eta_t}
=\mathcal{O}\left( \frac{\log T}{T^{\frac{\nu}{\nu+1}}}\right).
\end{align*}
In particular, when $\nu + 1 = \mathfrak{p}$,
\begin{align*}
&\min_{1\leq t\leq T} \mathbb{E}\left[ f(\bm{x}_t) - f^\star\right] 
=\mathcal{O}\left( \frac{\log T}{T^{\frac{\mathfrak{p}-1}{\mathfrak{p}}}}\right),
\end{align*}
where $E_4 := 2^{1-\nu} \left( E_0 + \frac{\nu+1}{2}L^2D_1^{2\nu} \right)$ and $E_0$ defined in Eq. \eqref{eq:E0} are non-negative constants.
\end{manualthm_sgd_c}

\begin{proof}
From Lemma \ref{lem:q-norm} and Assumption \ref{assum:03}(i), we have
\begin{align*}
\| \bm{x}_{t+1} - \bm{x}^\star \|^{\nu+1}
&= \| \bm{x}_t - \eta_t \nabla f_{\mathcal{S}_t}(\bm{x}_t) - \bm{x}^\star \|^{\nu+1} \\
&\leq\| \bm{x}_t - \bm{x}^\star \|^{\nu+1} - \eta_t(\nu+1)\| \bm{x}_t - \bm{x}^\star \|^{\nu-1}\langle \bm{x}_t - \bm{x}^\star, \nabla f_{\mathcal{S}_t}(\bm{x}_t)\rangle + 2^{1-\nu}\| \nabla f_{\mathcal{S}_t}(\bm{x}_t) \|^{\nu+1}.
\end{align*} 
Taking the conditional expectation of $\bm{x}_t$, we have, from the convexity of $f$ and Assumption \ref{assum:03}(i),
\begin{align*}
\mathbb{E}\left[ \| \bm{x}_{t+1} - \bm{x}^\star \|^{\nu+1} | \bm{x}_t\right]
&\leq \| \bm{x}_t - \bm{x}^\star \|^{\nu+1} - \eta_t(\nu+1)\| \bm{x}_t - \bm{x}^\star \|^{\nu-1}\langle \bm{x}_t - \bm{x}^\star, \nabla f(\bm{x}_t)\rangle + 2^{1-\nu} \mathbb{E}\left[ \| \nabla f_{\mathcal{S}_t}(\bm{x}_t) \|^{\nu+1} | \bm{x}_t \right] \\
&\leq \| \bm{x}_t - \bm{x}^\star \|^{\nu+1} - \eta_t(\nu+1)\| \bm{x}_t - \bm{x}^\star \|^{\nu-1}(f(\bm{x}_t) - f^\star) + 2^{1-\nu} \mathbb{E}\left[ \| \nabla f_{\mathcal{S}_t}(\bm{x}_t) \|^{\nu+1} | \bm{x}_t \right] \\
&\leq \| \bm{x}_t - \bm{x}^\star \|^{\nu+1} - \eta_t(\nu+1)D_1^{\nu-1}(f(\bm{x}_t) - f^\star) + 2^{1-\nu} \mathbb{E}\left[ \| \nabla f_{\mathcal{S}_t}(\bm{x}_t) \|^{\nu+1} | \bm{x}_t \right].
\end{align*}

By taking the total expectation, we have
\begin{align*}
\mathbb{E}\left[ \| \bm{x}_{t+1} - \bm{x}^\star \|^{\nu+1}\right]
\leq \mathbb{E}\left[ \| \bm{x}_t - \bm{x}^\star \|^{\nu+1}\right] - \eta_t(\nu+1)D_1^{\nu-1}(f(\bm{x}_t) - f^\star) + 2^{1-\nu}\eta_t^{\nu+1}\mathbb{E}\left[\| \nabla f_{\mathcal{S}_t}(\bm{x}_t) \|^{\nu+1} \right].
\end{align*}
From Lemma \ref{lem:aux}, we obtain
\begin{align*}
\mathbb{E}\left[ \| \bm{x}_{t+1} - \bm{x}^\star \|^{\nu+1}\right]
&\leq \mathbb{E}\left[ \| \bm{x}_t - \bm{x}^\star \|^{\nu+1}\right] - \eta_t(\nu+1)D_1^{\nu-1}\mathbb{E}\left[ f(\bm{x}_t) - f^\star \right] \\
&\quad+ 2^{1-\nu} \left\{ 2^{1-\nu} \left( \frac{C_\mathfrak{p} \sigma^\mathfrak{p}}{b^{\mathfrak{p}-1}} \right)^{\frac{\nu+1}{\mathfrak{p}}} + \frac{\nu+1}{2} \mathbb{E}\left[ \| \nabla f(\bm{x}_t) \|^2 \right] + \frac{1-\nu}{2} \right\}\eta_t^{\nu+1}.
\end{align*}
Here, from Assumptions \ref{assum:01} and \ref{assum:03}(i), we have $\| \nabla f(\bm{x}_t) - \nabla f(\bm{x}^\star) \| \leq L\| \bm{x}_t - \bm{x}^\star \|^\nu = LD_1^\nu$ , i.e., $\| \nabla f(\bm{x}_t) \|^2 \leq L^2D_1^{2\nu}$. Hence, 
\begin{align*}
\mathbb{E}\left[ \| \bm{x}_{t+1} - \bm{x}^\star \|^{\nu+1}\right]
&\leq \mathbb{E}\left[ \| \bm{x}_t - \bm{x}^\star \|^{\nu+1}\right] - \eta_t(\nu+1)D_1^{\nu-1}\mathbb{E}\left[ f(\bm{x}_t) - f^\star \right] \\
&\quad + \underbrace{2^{1-\nu} \left\{ 2^{1-\nu} \left( \frac{C_\mathfrak{p} \sigma^\mathfrak{p}}{b^{\mathfrak{p}-1}} \right)^{\frac{\nu+1}{\mathfrak{p}}} + \frac{\nu+1}{2}L^2D_1^{2\nu} + \frac{1-\nu}{2} \right\}}_{=:E_4}\eta_t^{\nu+1} \\
&\leq \mathbb{E}\left[ \| \bm{x}_t - \bm{x}^\star \|^{\nu+1}\right] - \eta_t(\nu+1)D_1^{\nu-1}\mathbb{E}\left[ f(\bm{x}_t) - f^\star \right] + E_4 \eta_t^{\nu+1}.
\end{align*}
Rearranging the terms, 
\begin{align*}
\eta_t \mathbb{E}\left[ f(\bm{x}_t) - f^\star\right]
\leq \frac{1}{(\nu+1)D_1^{\nu-1}}\left( \mathbb{E}\left[ \| \bm{x}_t - \bm{x}^\star \|^{\nu+1} \right] - \mathbb{E}\left[ \| \bm{x}_{t+1} - \bm{x}^\star \|^{\nu+1}\right]\right) + \frac{E_4}{(\nu+1)D_1^{\nu-1}}\eta_t^{\nu+1},
\end{align*}
and summing over $t=1$ to $t=T$, we have 
\begin{align*}
\sum_{t=1}^{T} \eta_t \mathbb{E}\left[ f(\bm{x}_t) - f^\star\right]
\leq \frac{\| \bm{x}_1 - \bm{x}^\star \|^{\nu+1}}{(\nu+1)D_1^{\nu-1}} + \frac{E_4}{(\nu+1)D_1^{\nu-1}} \sum_{t=1}^{T}\eta_t^{\nu+1}.
\end{align*}
Hence, 
\begin{align*}
\min_{1\leq t\leq T} \mathbb{E}\left[ f(\bm{x}_t) - f^\star\right]
\leq \frac{\| \bm{x}_1 - \bm{x}^\star \|^{\nu+1}}{(\nu+1)D_1^{\nu-1}}\frac{1}{\sum_{t=1}^{T}\eta_t} + \frac{E_4}{(\nu+1)D_1^{\nu-1}} \frac{\sum_{t=1}^{T}\eta_t^{\nu+1}}{\sum_{t=1}^{T}\eta_t}.
\end{align*}
Since $\eta_t := \eta_{\max}t^{-\frac{1}{\nu+1}}$, from Lemma \ref{lem:lr}, this completes the proof.
\end{proof}

\subsection{Proof of Theorem \ref{thm:sgdm_c}}
\label{sec:sgdm_c}
\begin{manualthm_sgdm_c}
Let the function $f \colon \mathbb{R}^d \to \mathbb{R}$ be convex. Suppose that Assumptions \ref{assum:01}, \ref{assum:02}, \ref{assum:03}(i), and \ref{assum:03}(ii) hold. With learning rate $\eta_t := \eta$, if $\nu+1 \leq \mathfrak{p}$, then:
\begin{align*}
&\min_{1 \leq t \leq T} f(\bm{x}_t) - f^\star 
\leq \frac{\| \bm{x}_1 - \bm{x}^\star\|^{\nu+1}}{\eta(\nu+1)D_2^{\nu-1}T} + E_6\eta + \frac{2^{1-\nu}E_5}{(\nu+1)D_2^{\nu-1}}\eta^\nu 
=\mathcal{O}\left( \frac{1}{\eta T} + \eta^\nu\right).
\end{align*}
In particular, when $\nu + 1 = \mathfrak{p}$,
\begin{align*}
&\min_{1 \leq t \leq T} f(\bm{x}_t) - f^\star 
=\mathcal{O}\left( \frac{1}{\eta T} + \eta^{\mathfrak{p}-1}\right),
\end{align*}
where $E_5 := 2^{1-\nu}\left( \frac{C_\mathfrak{p} \sigma^\mathfrak{p}}{b^{\mathfrak{p}-1}}\right)^{\frac{\nu+1}{\mathfrak{p}}} + L^{\nu+1}D_1^{\nu(\nu+1)}$, $E_6 := \frac{\beta\left(E_5 + \nu L^{\frac{\nu+1}{\nu}}D_1^{\nu+1} \right)}{(1-\beta)(\nu+1)}$ are non-negative constants.
\end{manualthm_sgdm_c}

\begin{proof}
For all $t \in \mathbb{N}$, let $w_t := \| \bm{z}_t - \bm{x}^\star \|^{\nu-1}$. From Lemma \ref{lem:q-norm},
\begin{align*}
\| \bm{z}_{t+1} - \bm{x}^\star \|^{\nu+1}
&= \| \bm{z}_t - \bm{x}^\star - \eta \nabla f_{\mathcal{S}_t}(\bm{x}_t) \|^{\nu+1} \\
&\leq \| \bm{z}_t - \bm{x}^\star \|^{\nu+1} - \eta(\nu+1)w_t\langle \bm{z}_t - \bm{x}^\star, \nabla f_{\mathcal{S}_t}(\bm{x}_t)\rangle + 2^{1-\nu}\eta^{\nu+1}\| \nabla f_{\mathcal{S}_t}(\bm{x}_t)\|^{\nu+1} \\
&=\| \bm{z}_t - \bm{x}^\star \|^{\nu+1} - \eta(\nu+1)w_t\langle \bm{x}_t - \bm{x}^\star, \nabla f_{\mathcal{S}_t}(\bm{x}_t)\rangle \\
&\quad+ \eta(\nu+1)w_t \langle\bm{x}_t - \bm{z}_t , \nabla f_{\mathcal{S}_t}(\bm{x}_t) \rangle+ 2^{1-\nu}\eta^{\nu+1}\| \nabla f_{\mathcal{S}_t}(\bm{x}_t)\|^{\nu+1} \\
&= \| \bm{z}_t - \bm{x}^\star \|^{\nu+1} - \eta(\nu+1)w_t\langle \bm{x}_t - \bm{x}^\star, \nabla f_{\mathcal{S}_t}(\bm{x}_t)\rangle \\
&\quad+ \frac{\eta^2\beta(\nu+1)}{1-\beta}w_t \langle\bm{m}_{t-1} , \nabla f_{\mathcal{S}_t}(\bm{x}_t) \rangle+ 2^{1-\nu}\eta^{\nu+1}\| \nabla f_{\mathcal{S}_t}(\bm{x}_t)\|^{\nu+1}.
\end{align*}
Taking the conditional expectation of $\bm{x}_t$, we have, from the convexity of $f$,
\begin{align*}
\mathbb{E}\left[\| \bm{z}_{t+1} - \bm{x}^\star \|^{\nu+1} | \bm{x}_t \right]
&\leq \| \bm{z}_t - \bm{x}^\star \|^{\nu+1} - \eta(\nu+1)w_t(f(\bm{x}_t) - f^\star) \\
&\quad+ \frac{\eta^2\beta(\nu+1)}{1-\beta}w_t \langle\bm{m}_{t-1} , \nabla f(\bm{x}_t) \rangle+ 2^{1-\nu}\eta^{\nu+1}\mathbb{E}\left[ \| \nabla f_{\mathcal{S}_t}(\bm{x}_t)\|^{\nu+1} | \bm{x}_t \right].
\end{align*}
Moreover, from the Cauchy-Schwarz inequality, Young's inequality, and Assumptions \ref{assum:01} and \ref{assum:03}(i), 
\begin{align*}
\langle \bm{m}_{t-1}, \nabla f(\bm{x}_t)\rangle
\leq \| \bm{m}_{t-1} \| \| \nabla f(\bm{x}_t) \|
\leq \frac{\| \bm{m}_{t-1} \|^{\nu+1}}{\nu+1} + \frac{\nu \| \nabla f(\bm{x}_t) \|^{\frac{\nu+1}{\nu}}}{\nu+1}
\leq \frac{\| \bm{m}_{t-1} \|^{\nu+1}}{\nu+1} + \frac{\nu L^{\frac{\nu+1}{\nu}}D_1^{\nu+1}}{\nu+1}.
\end{align*}
In addition, from Lemma \ref{lem:aux} and the convexity of $\| \cdot \|^{\nu+1}$,
\begin{align*}
\mathbb{E}\left[ \| \bm{m}_{t-1} \|^{\nu+1}\right]
&\leq \beta \mathbb{E}\left[ \| \bm{m}_{t-1} \|^{\nu+1}\right] + (1-\beta)\mathbb{E}\left[ \| \nabla f_{\mathcal{S}_t}(\bm{x}_t) \|^{\nu+1}\right]\\
&\leq \beta \mathbb{E}\left[ \| \bm{m}_{t-1} \|^{\nu+1}\right] + (1-\beta)\left\{ 2^{1-\nu}\left( \frac{C_\mathfrak{p} \sigma^\mathfrak{p}}{b^{\mathfrak{p}-1}}\right)^{\frac{\nu+1}{\mathfrak{p}}} + \mathbb{E}\left[ \| \nabla f(\bm{x}_t) \|^{\nu+1}\right]\right\} \\
&\leq \beta \mathbb{E}\left[ \| \bm{m}_{t-1} \|^{\nu+1}\right] + (1-\beta)\left\{ 2^{1-\nu}\left( \frac{C_\mathfrak{p} \sigma^\mathfrak{p}}{b^{\mathfrak{p}-1}}\right)^{\frac{\nu+1}{\mathfrak{p}}} + L^{\nu+1}D_1^{\nu(\nu+1)}\right\} \\ 
&\leq \underbrace{2^{1-\nu}\left( \frac{C_\mathfrak{p} \sigma^\mathfrak{p}}{b^{\mathfrak{p}-1}}\right)^{\frac{\nu+1}{\mathfrak{p}}} + L^{\nu+1}D_1^{\nu(\nu+1)}}_{=:E_5}.
\end{align*}
Therefore, by taking the expectation, we have
\begin{align*}
\mathbb{E}\left[ \| \bm{z}_{t+1} - \bm{x}^\star \|^{\nu+1}\right]
\leq \mathbb{E}\left[ \| \bm{z}_t - \bm{x}^\star \|^{\nu+1}\right] - \eta(\nu+1)w_t (f(\bm{x}_t) - f^\star) + \frac{\eta^2\beta w_t}{1-\beta}\left(E_5 + \nu L^{\frac{\nu+1}{\nu}}D_1^{\nu+1} \right) + 2^{1-\nu}\eta^{\nu+1}E_5.
\end{align*}
Rearranging the terms, 
\begin{align*}
w_t(f(\bm{x}_t) -f^\star)
\leq \frac{\mathbb{E}\left[ \| \bm{z}_t - \bm{x}^\star \|^{\nu+1}\right] - \mathbb{E}\left[ \| \bm{z}_{t+1} - \bm{x}^\star \|^{\nu+1}\right]}{\eta(\nu+1)} + \frac{\beta\left(E_5 + \nu L^{\frac{\nu+1}{\nu}}D_1^{\nu+1} \right)}{(1-\beta)(\nu+1)}\eta w_t + \frac{2^{1-\nu}E_5}{\nu+1}\eta^\nu,
\end{align*}
and summing over $t=1$ to $t=T$, we obtain 
\begin{align*}
\sum_{t=1}^{T} w_t(f(\bm{x}_t)-f^\star)
\leq \frac{\| \bm{x}_1 - \bm{x}^\star\|^{\nu+1}}{\eta(\nu+1)} +  \frac{\beta\left(E_5 + \nu L^{\frac{\nu+1}{\nu}}D_1^{\nu+1} \right)}{(1-\beta)(\nu+1)}\eta \sum_{t=1}^{T}w_t + \frac{2^{1-\nu}E_5}{\nu+1}\eta^\nu T.
\end{align*}
Hence,
\begin{align*}
\min_{1 \leq t \leq T} f(\bm{x}_t) - f^\star
\leq \frac{\sum_{t=1}^{T} w_t(f(\bm{x}_t)-f^\star)}{\sum_{t=1}^{T}w_t}
\leq \frac{\| \bm{x}_1 - \bm{x}^\star\|^{\nu+1}}{\eta(\nu+1) \sum_{t=1}^{T}w_t} + \frac{\beta\left(E_5 + \nu L^{\frac{\nu+1}{\nu}}D_1^{\nu+1} \right)}{(1-\beta)(\nu+1)}\eta + \frac{2^{1-\nu}E_5}{(\nu+1)\sum_{t=1}^{T}w_t}\eta^\nu T.
\end{align*}
Here, from Assumption \ref{assum:03}(ii), $w_t := \| \bm{z}_t - \bm{x}^\star \|^{\nu-1} \geq D_2^{\nu-1}$. Therefore,
\begin{align*}
\min_{1 \leq t \leq T} f(\bm{x}_t) - f^\star
&\leq \frac{\| \bm{x}_1 - \bm{x}^\star\|^{\nu+1}}{\eta(\nu+1)D_2^{\nu-1}T} + \underbrace{\frac{\beta\left(E_5 + \nu L^{\frac{\nu+1}{\nu}}D_1^{\nu+1} \right)}{(1-\beta)(\nu+1)}}_{=:E_6}\eta + \frac{2^{1-\nu}E_5}{(\nu+1)D_2^{\nu-1}}\eta^\nu \\
&=\mathcal{O}\left( \frac{1}{\eta T} + \eta^\nu\right).
\end{align*}
This completes the proof.
\end{proof}

\subsection{Proof of Theorem \ref{thm:sgd_n}}
\label{sec:sgd_n}
\begin{manualthm_sgd_n}
Let the function $f \colon \mathbb{R}^d \to \mathbb{R}$ be nonconvex. Suppose that Assumptions \ref{assum:01} and \ref{assum:02} hold. With learning rate $\eta_t := \eta_{\max}t^{-\frac{1}{\nu+1}}$, if $\nu+1 \leq \mathfrak{p}$ and $\eta_{t}^\nu < \frac{2}{L}$ for all $t \in [T]$, then: 
\begin{align*}
&\min_{1 \leq t \leq T}\mathbb{E}\left[\| \nabla f(\bm{x}_t) \|^2 \right] 
\leq \frac{ 2(f(\bm{x}_1) - f^\star)}{\left( 2- L\eta_{\max}^{\nu}\right)} \frac{1}{\sum_{t=1}^{T} \eta_t} + E_7 \frac{\sum_{t=1}^{T} \eta_t^{\nu+1}}{ \sum_{t=1}^{T} \eta_t} 
= \mathcal{O}\left(\frac{\log T}{T^{\frac{\nu}{\nu+1}}} \right).
\end{align*}
In particular, when $\nu + 1 = \mathfrak{p}$,
\begin{align*}
&\min_{1 \leq t \leq T}\mathbb{E}\left[\| \nabla f(\bm{x}_t) \|^2 \right] 
= \mathcal{O}\left(\frac{\log T}{T^{\frac{\mathfrak{p}-1}{\mathfrak{p}}}} \right),
\end{align*}
where $E_7 := \frac{2E_0L}{(\nu+1)\left( 2-L\eta_{\max}^\nu\right)}$ and $E_0$ defined in Eq. \eqref{eq:E0} are non-negative constants.
\end{manualthm_sgd_n}
\begin{proof}
From Lemma \ref{lem:holder}, we have
\begin{align*}
f(\bm{x}_{t+1})
&\leq f(\bm{x}_t) + \langle \nabla f(\bm{x}_t), \bm{x}_{t+1}-\bm{x}_t \rangle + \frac{L}{\nu + 1} \| \bm{x}_{t+1} - \bm{x}_t \|^{\nu+1} \\
&= f(\bm{x}_t) - \eta_t \langle \nabla f(\bm{x}_t), \nabla f_{\mathcal{S}_t}(\bm{x}_t)\rangle + \frac{L \eta_t^{\nu+1}}{\nu+1} \| \nabla f_{\mathcal{S}_t}(\bm{x}_t) \|^{\nu+1}.
\end{align*}
By taking the expectation, we have
\begin{align*}
\mathbb{E}\left[ f(\bm{x}_{t+1}) \right]
&\leq \mathbb{E}\left[ f(\bm{x}_t) \right] -\eta_t \mathbb{E}\left[ \| \nabla f(\bm{x}_t) \|^2\right] + \frac{L \eta_t^{\nu+1}}{\nu+1} \mathbb{E}\left[ \| \nabla f_{\mathcal{S}_t}(\bm{x}_t) \|^{\nu+1} \right].
\end{align*}
From Lemma \ref{lem:aux}, 
\begin{align*}
\mathbb{E}\left[ f(\bm{x}_{t+1}) \right]
&\leq \mathbb{E}\left[ f(\bm{x}_t) \right] -\eta_t \mathbb{E}\left[ \| \nabla f(\bm{x}_t) \|^2\right] + \frac{L \eta_t^{\nu+1}}{\nu+1} \left\{ 2^{1-\nu}\left( \frac{C_\mathfrak{p} \sigma^\mathfrak{p}}{b^{\mathfrak{p}-1}} \right)^{\frac{\nu+1}{\mathfrak{p}}} + \frac{\nu+1}{2} \mathbb{E}\left[ \| \nabla f(\bm{x}_t) \|^2 \right] + \frac{1-\nu}{2} \right\} \\
&= \mathbb{E}\left[ f(\bm{x}_t) \right] - \eta_t\left( 1- \frac{L\eta_t^{\nu}}{2} \right) \mathbb{E}\left[ \| \nabla f(\bm{x}_t) \|^2\right] + \frac{L \eta_t^{\nu+1}}{\nu+1} \left\{ 2^{1-\nu} \left( \frac{C_\mathfrak{p} \sigma^\mathfrak{p}}{b^{\mathfrak{p}-1}} \right)^{\frac{\nu+1}{\mathfrak{p}}} + \frac{1-\nu}{2} \right\} \\
&\leq \mathbb{E}\left[ f(\bm{x}_t) \right] - \eta_t\left( 1- \frac{L\eta_{\max}^{\nu}}{2} \right) \mathbb{E}\left[ \| \nabla f(\bm{x}_t) \|^2\right] + \frac{L \eta_t^{\nu+1}}{\nu+1} \left\{ 2^{1-\nu}\left( \frac{C_\mathfrak{p} \sigma^\mathfrak{p}}{b^{\mathfrak{p}-1}} \right)^{\frac{\nu+1}{\mathfrak{p}}} + \frac{1-\nu}{2} \right\}.
\end{align*}
Rearranging the terms,
\begin{align*}
\eta_t\left( 1- \frac{L\eta_{\max}^{\nu}}{2} \right) \mathbb{E}\left[ \| \nabla f(\bm{x}_t) \|^2\right]
&\leq \mathbb{E}\left[ f(\bm{x}_t) \right] - \mathbb{E}\left[ f(\bm{x}_{t+1}) \right] +\frac{L \eta_t^{\nu+1}}{\nu+1} \left\{ 2^{1-\nu}\left( \frac{C_\mathfrak{p} \sigma^\mathfrak{p}}{b^{\mathfrak{p}-1}} \right)^{\frac{\nu+1}{\mathfrak{p}}} + \frac{1-\nu}{2} \right\},
\end{align*}
and summing over $t = 1$ to $t= T$, we have
\begin{align*}
\left( 1- \frac{L\eta_{\max}^{\nu}}{2} \right)\sum_{t=1}^{T} \eta_t \mathbb{E}\left[ \| \nabla f(\bm{x}_t) \|^2\right]
\leq f(\bm{x}_1) - f^\star + \frac{L}{\nu+1} \underbrace{\left\{ 2^{1-\nu}\left( \frac{C_\mathfrak{p} \sigma^\mathfrak{p}}{b^{\mathfrak{p}-1}} \right)^{\frac{\nu+1}{\mathfrak{p}}} + \frac{1-\nu}{2} \right\}}_{=:E_0} \sum_{t=1}^{T} \eta_t^{\nu+1}.
\end{align*}
Therefore, 
\begin{align*}
\min_{1 \leq t \leq T}\mathbb{E}\left[\| \nabla f(\bm{x}_t) \|^2 \right]
\leq \frac{ 2(f(\bm{x}_1) - f^\star)}{\left( 2- L\eta_{\max}^{\nu}\right)} \frac{1}{\sum_{t=1}^{T} \eta_t} + \underbrace{\frac{2E_0L}{(\nu+1)\left( 2-L\eta_{\max}^\nu\right)}}_{=:E_7} \frac{\sum_{t=1}^{T} \eta_t^{\nu+1}}{ \sum_{t=1}^{T} \eta_t}.
\end{align*}
Since $\eta_t := \eta_{\max}t^{-\frac{1}{\nu+1}}$, we can use Lemma \ref{lem:lr} to complete the proof.
\end{proof}

\subsection{Proof of Theorem \ref{thm:sgdm_n}}
\label{sec:sgdm_n}
\begin{manualthm_sgdm_n}
Let the function $f \colon \mathbb{R}^d \to \mathbb{R}$ be nonconvex. Suppose that Assumptions \ref{assum:01} and \ref{assum:02} hold. With learning rate $\eta_t := \eta$, if $\nu+1 \leq \mathfrak{p}$ and $\eta \le \min \left\{ \left( \frac{1}{4L} \right)^{\frac{1}{\nu}}, \frac{1-\beta}{\beta} \left( \frac{1}{4\nu L^2} \right)^{\frac{1}{2\nu}} \right\}$, then: 
\begin{align*}
&\min_{1 \leq t \leq T}\mathbb{E}\left[\| \nabla f(\bm{x}_t) \|^2 \right] 
\leq \frac{4(f(\bm{x}_1) - f^\star)}{\eta T} + \frac{LE_0\eta^{\nu}}{4(\nu+1)} + E_8\eta^{2\nu}
=\mathcal{O}\left( \frac{1}{\eta T} + \eta^{\nu}\right).
\end{align*}
In particular, when $\nu + 1 = \mathfrak{p}$,
\begin{align*}
&\min_{1 \leq t \leq T}\mathbb{E}\left[\| \nabla f(\bm{x}_t) \|^2 \right] 
= \mathcal{O}\left( \frac{1}{\eta T} + \eta^{\mathfrak{p}-1}\right),
\end{align*}
where $E_8 := \frac{2L}{\nu+1}\left(\frac{\beta}{1-\beta} \right)^{2\nu}\left\{L(1-\nu)+2E_0\nu \right\}$ and $E_0$ defined in Eq. \eqref{eq:E0} are non-negative constants.
\end{manualthm_sgdm_n}

\begin{proof}
For all $t \in \mathbb{N}$, let us define $\Phi_t := f(\bm{z}_t) + \lambda \| \bm{m}_{t-1} \|^{\nu+1}$, $\lambda:=\frac{4\eta L \nu}{(1-\beta)(\nu+1)}\left( \frac{\eta\beta}{1-\beta}\right)^{2\nu}$. Then, we have
\begin{align*}
\mathbb{E}[\Phi_{t+1}] - \mathbb{E}[\Phi_{t}] = \mathbb{E}[f(\bm{z}_{t+1})] - \mathbb{E}[f(\bm{z}_t)] + \lambda \left( \mathbb{E}[\| \bm{m}_t \|^{\nu+1}] - \mathbb{E}[\| \bm{m}_{t-1} \|^{\nu+1}] \right).
\end{align*}
Therefore, from Lemmas \ref{lem:sgdm_s_01} and \ref{lem:sgdm_s_02},
\begin{align*}
&\mathbb{E}[\Phi_{t+1}] - \mathbb{E}[\Phi_t] \\
&\quad\le \mathbb{E}[f(\bm{z}_t)] -\left( \frac{\eta}{2} - \frac{L\eta^{\nu+1}}{2}\right)\mathbb{E}[\| \nabla f(\bm{x}_t) \|^2] + \frac{\eta L^2 \nu}{\nu+1}\left( \frac{\eta\beta}{1-\beta}\right)^{2\nu}\mathbb{E}\left[\| \bm{m}_{t-1} \|^{\nu+1} \right] \\
&\quad\quad+ \frac{\eta(1-\nu)L^2}{2(\nu+1)}\left( \frac{\eta\beta}{1-\beta}\right)^{2\nu}
+\frac{L\eta^{\nu+1}}{\nu+1} \underbrace{\left\{ 2^{1-\nu} \left( \frac{C_\mathfrak{p} \sigma^\mathfrak{p}}{b^{\mathfrak{p}-1}} \right)^{\frac{\nu+1}{\mathfrak{p}}} + \frac{1-\nu}{2} \right\}}_{=:E_0} \\
&\quad\quad+ \lambda \left\{ -(1-\beta)\mathbb{E}[\| \bm{m}_{t-1} \|^{\nu+1}] + \frac{(1-\beta)(1+\nu)}{2} \mathbb{E}[\| \nabla f(\bm{x}_t) \|^2] + \frac{(1-\beta)(1-\nu)}{2} + 2^{1-\nu}(1-\beta)^{\nu+1}\left( \frac{C_\mathfrak{p} \sigma^\mathfrak{p}}{b^{\mathfrak{p}-1}}\right)^{\frac{\nu+1}{\mathfrak{p}}} \right\} \\
&\quad= -\frac{\eta}{2} \left( 1 - L\eta^\nu - L^2 \nu \left( \frac{\eta\beta}{1-\beta} \right)^{2\nu} \right) \mathbb{E}[\| \nabla f(\bm{x}_t) \|^2]
+ \left( \frac{\eta L^2 \nu}{\nu+1}\left( \frac{\eta\beta}{1-\beta}\right)^{2\nu} - \lambda(1-\beta) \right) \mathbb{E}[\| \bm{m}_{t-1} \|^{\nu+1}] \\
&\quad\quad + \frac{LE_0\eta^{\nu+1}}{\nu+1} + \frac{\eta(1-\nu)L^2}{2(\nu+1)}\left( \frac{\eta\beta}{1-\beta}\right)^{2\nu} + \lambda \left\{ \frac{(1-\beta)(1-\nu)}{2} + 2^{1-\nu}(1-\beta)^{\nu+1}\left( \frac{C_\mathfrak{p} \sigma^\mathfrak{p}}{b^{\mathfrak{p}-1}}\right)^{\frac{\nu+1}{\mathfrak{p}}} \right\} \\
&\quad\le -\frac{\eta}{2} \left( 1 - L\eta^\nu - L^2 \nu \left( \frac{\eta\beta}{1-\beta} \right)^{2\nu} \right) \mathbb{E}[\| \nabla f(\bm{x}_t) \|^2] 
+ \frac{LE_0\eta^{\nu+1}}{\nu+1} \\
&\quad\quad+ \frac{\eta(1-\nu)L^2}{2(\nu+1)}\left( \frac{\eta\beta}{1-\beta}\right)^{2\nu}
+ \frac{4\eta L \nu}{(\nu+1)}\left( \frac{\eta\beta}{1-\beta}\right)^{2\nu} E_0,
\end{align*}
where we have used the definition of $\lambda$ and $(1-\beta)^{\nu+1} <1-\beta$ in the last inequality. Here, from the condition on $\eta$, 
\begin{align*}
\frac{\eta}{2} \left( 1 - L\eta^\nu - L^2 \nu \left( \frac{\eta\beta}{1-\beta} \right)^{2\nu} \right) \geq \frac{\eta}{2} \cdot \frac{1}{2} = \frac{\eta}{4}.
\end{align*}
Hence, we have
\begin{align*}
\frac{\eta}{4} \mathbb{E}[\| \nabla f(\bm{x}_t) \|^2] 
&\le \mathbb{E}[\Phi_t] - \mathbb{E}[\Phi_{t+1}] + \frac{LE_0\eta^{\nu+1}}{\nu+1} + \frac{\eta(1-\nu)L^2}{2(\nu+1)}\left( \frac{\eta\beta}{1-\beta}\right)^{2\nu} + \frac{4\eta L \nu}{(\nu+1)}\left( \frac{\eta\beta}{1-\beta}\right)^{2\nu} E_0.
\end{align*}
Summing over $t = 1$ to $t=T$, we have
\begin{align*}
\sum_{t=1}^T \mathbb{E}[\| \nabla f(\bm{x}_t) \|^2] 
&\leq \frac{4(f(\bm{x}_1) - f^\star)}{\eta} + \frac{LE_0\eta^{\nu}}{4(\nu+1)}T + \frac{2(1-\nu)L^2}{(\nu+1)}\left( \frac{\eta\beta}{1-\beta}\right)^{2\nu}T + \frac{4 L \nu}{(\nu+1)}\left( \frac{\eta\beta}{1-\beta}\right)^{2\nu} E_0 T.
\end{align*}
Therefore,
\begin{align*}
\frac{1}{T} \sum_{t=1}^T \mathbb{E}[\| \nabla f(\bm{x}_t) \|^2] 
&\leq \frac{4(f(\bm{x}_1) - f^\star)}{\eta T} + \frac{LE_0}{4(\nu+1)}\eta^{\nu} + \frac{2(1-\nu)L^2}{(\nu+1)}\left( \frac{\beta}{1-\beta}\right)^{2\nu}\eta^{2\nu} + \frac{4 L \nu}{(\nu+1)}\left( \frac{\beta}{1-\beta}\right)^{2\nu} E_0\eta^{2\nu} \\
&=\frac{4(f(\bm{x}_1) - f^\star)}{\eta T} + \frac{LE_0}{4(\nu+1)}\eta^{\nu} + \underbrace{\frac{2L}{\nu+1}\left(\frac{\beta}{1-\beta} \right)^{2\nu}\left\{L(1-\nu)+2E_0\nu \right\}}_{=:E_8}\eta^{2\nu} \\
&= \mathcal{O}\left( \frac{1}{\eta T} + \eta^{\nu}\right).
\end{align*}
Since the minimum value is smaller than the average value, we have
\begin{align*}
\min_{1\leq t\leq T} \mathbb{E}\left[\| \nabla f(\bm{x}_t) \|^2 \right]
\leq \frac{1}{T}\sum_{t=1}^{T}\mathbb{E}\left[ \| \nabla f(\bm{x}_t) \|^2\right].
\end{align*}
This completes the proof.
\end{proof}

\section{Additional Experiments}
\label{sec:addexp}
In this section, we present numerical experimental results for vanilla SGD. These results parallel the observations for vanilla SGD with momentum discussed in Section \ref{sec:exp}. All results presented in this work are fully reproducible via the source code available at: \url{https://github.com/iiduka-researches/vanillaSGDM_uai2026}. The code is designed to run directly on Google Colab for ease of use.

Figure \ref{fig:heatmap_sgd} displays the convergence performance of vanilla SGD across various combinations of $\nu$ and $\alpha$. Similar to the SGD with momentum case, the high-performance configurations (marked in red) are concentrated in the lower-right region where the theoretical condition $\nu+1 \leq \alpha$ is satisfied. This further confirms that our stability boundary is a fundamental requirement for the vanilla stochastic gradient descent family, regardless of the momentum parameter.

Figure \ref{fig:loss_sgd} illustrates the convergence behavior of vanilla SGD over 1,000 steps for fixed $\nu$. The plots exhibit the same characteristics evident in Figure \ref{fig:loss_m}: a smaller $\alpha$ leads to both degraded average performance and significantly higher variance across trials. Consistent with our theory, trajectories for larger $\nu$ become increasingly erratic as $\alpha$ decreases, whereas smaller $\nu$ values provide a more robust optimization path.
\begin{figure*}[!t]
\centering
\includegraphics[width=0.95\linewidth]{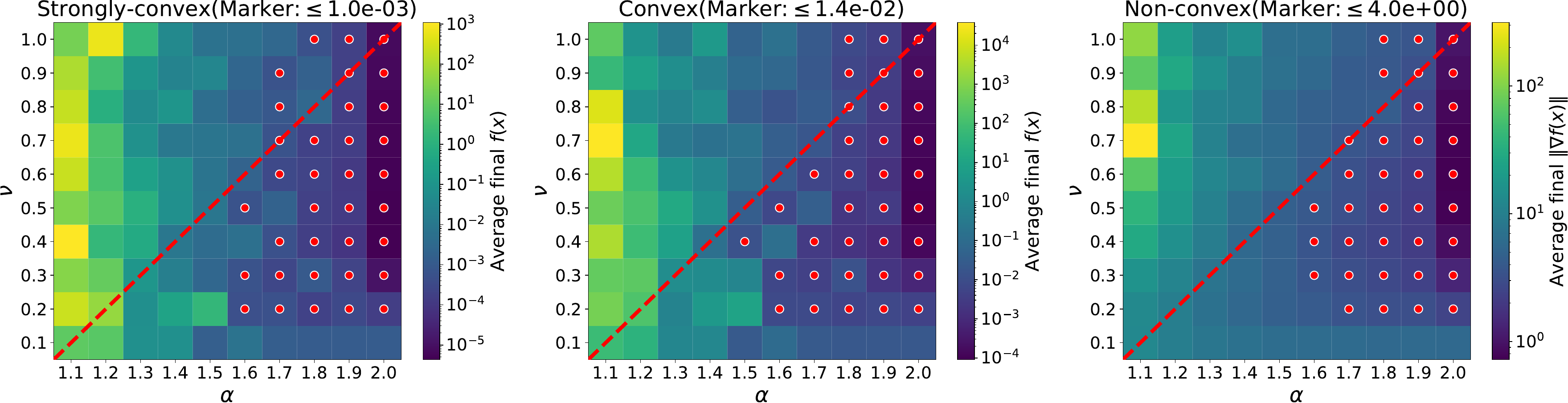}
\caption{Convergence performance across various $(\nu, \alpha)$ pairs. The heatmaps show the average final objective values (for strongly convex and convex cases) and gradient norms (for nonconvex cases) after 4,000 steps of vanilla SGD. The top 35 performing combinations are indicated by red markers. Below the red dotted line, the theoretical condition $\nu+1 \leq \alpha$ is satisfied, ensuring convergence.}
\label{fig:heatmap_sgd}
\end{figure*}

\begin{figure*}[!t]
\centering
\includegraphics[width=0.9\linewidth]{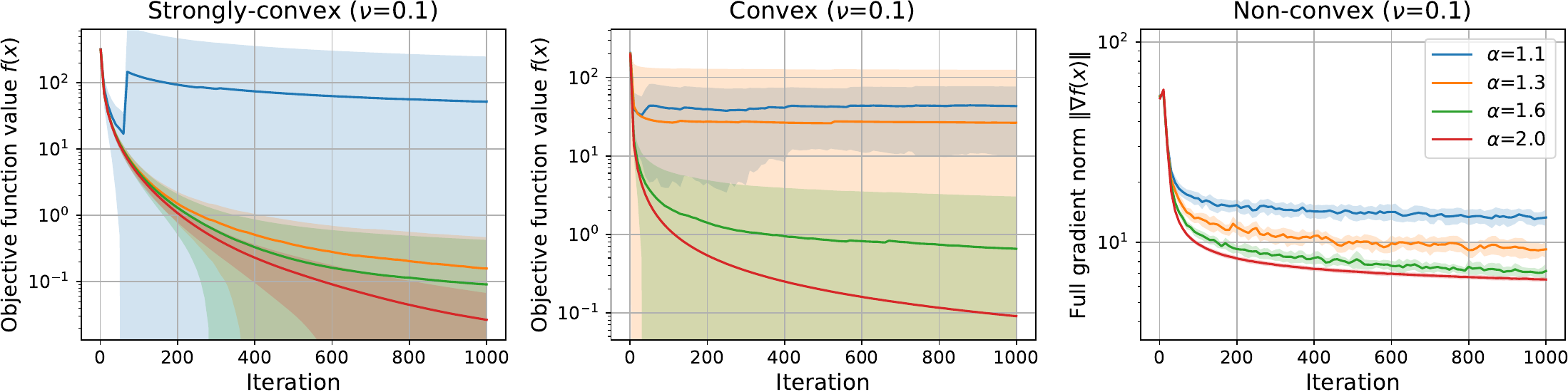}
\includegraphics[width=0.9\linewidth]{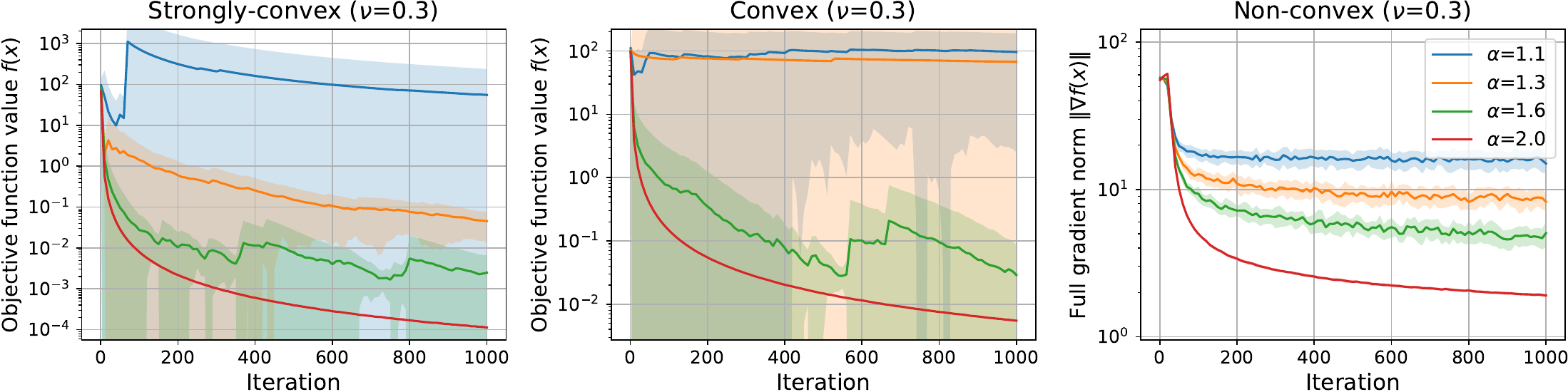}
\includegraphics[width=0.9\linewidth]{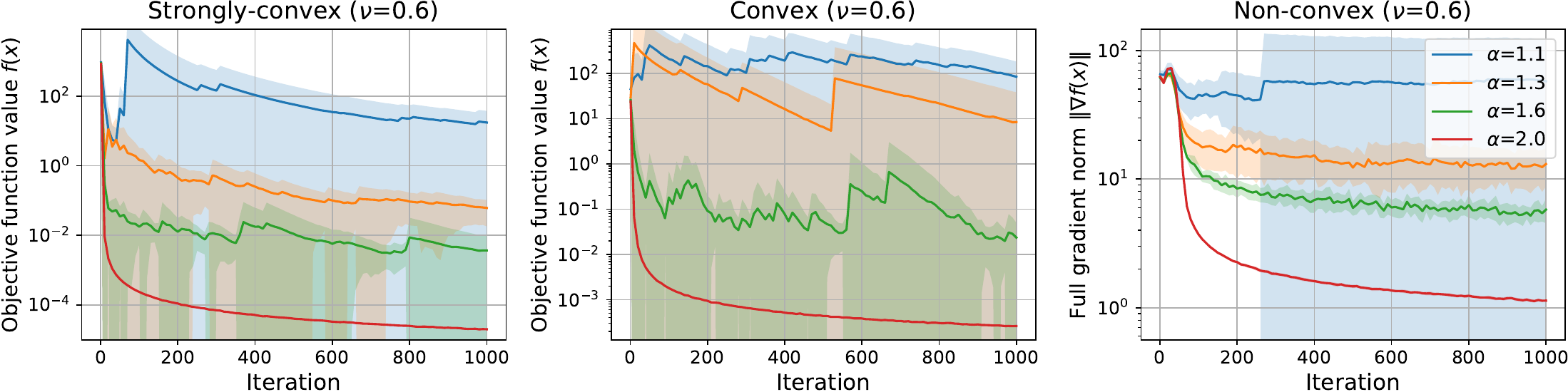}
\includegraphics[width=0.9\linewidth]{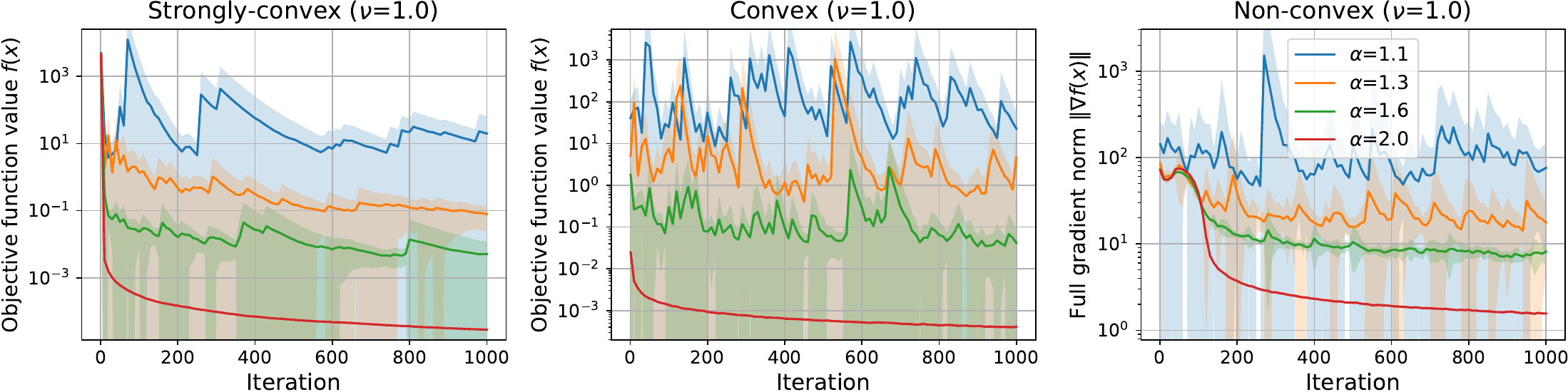}
\caption{Convergence trajectories of vanilla SGD. The plots illustrate the evolution of objective values (for strongly convex and convex cases) and gradient norms (for nonconvex cases) over 1,000 steps for the functions defined in Eq. (\ref{eq:function}). Solid lines represent the average of 20 independent trials, while the lightly shaded regions indicate the range between the minimum and maximum values.}
\label{fig:loss_sgd}
\end{figure*}

\end{document}